\documentclass[aos,preprint]{imsart}

\RequirePackage[OT1]{fontenc}
\RequirePackage{amsthm,amsmath}
\RequirePackage[numbers]{natbib}
\RequirePackage[colorlinks,citecolor=blue,urlcolor=blue]{hyperref}
\RequirePackage[colorinlistoftodos]{todonotes}
\RequirePackage{mathrsfs}

\usepackage{amssymb,amsfonts}
\usepackage{amsmath}
\usepackage{amsthm}
\usepackage{dsfont}
\usepackage[all,arc]{xy}
\usepackage{enumerate}
\usepackage{mathrsfs}
\usepackage{stmaryrd} 
\usepackage[colorinlistoftodos]{todonotes}
\usepackage{multirow}
\usepackage{subcaption}
\usepackage{placeins}
\usepackage{float}
\usepackage{booktabs}
\usepackage{empheq}

\arxiv{arXiv:0000.0000}

\startlocaldefs

\numberwithin{equation}{section}
\theoremstyle{plain}
\newtheorem{thm}{Theorem}[section]

\newtheorem{prop}[thm]{Proposition}
\newtheorem{lem}[thm]{Lemma}

\newtheorem{ass}[thm]{Assumption}
\newtheorem{condition}[thm]{Condition}

\newtheorem{defn}[thm]{Definition}

\newtheorem{rem}[thm]{Remark}

\newcommand{\Proba}{\mathds{P}}

\newcommand{\R}{\mathds{R}}

\newcommand{\E}{\mathds{E}}
\DeclareMathOperator*{\argmin}{argmin}

\DeclareMathOperator*{\supp}{supp}

\renewcommand*{\ast}{\star}

\DeclarePairedDelimiter{\nor}{\lVert}{\rVert}

\newcommand{\norm}[1]{\left\Vert #1\right\Vert}
\newcommand{\abs}[1]{\left\vert #1\right\vert}

\endlocaldefs

\begin{document}

\begin{frontmatter}
\title{High-dimensional robust regression and outliers detection with SLOPE}
\runtitle{SLOPE for outliers detection}

\begin{aug}
\author{\fnms{Alain} \snm{Virouleau}\thanksref{t1,m1}\ead[label=e1]{alain.virouleau@polytechnique.edu}},
\author{\fnms{Agathe} \snm{Guilloux}\thanksref{m1,m2}\ead[label=e2]{agathe.guilloux@univ-evry.fr}},
\author{\fnms{St\'ephane} \snm{Ga\"iffas}\thanksref{t1,m4,m1}\ead[label=e3]{stephane.gaiffas@polytechnique.edu}}
\and
\author{\fnms{Malgorzata} \snm{Bogdan}\thanksref{m3}
\ead[label=e4]{malgorzata.bogdan@uwr.edu.pl}}

\thankstext{t1}{Data Science Initiative of Ecole polytechnique}

\runauthor{A. Virouleau, A. Guilloux, S. Ga\"iffas and M. Bogdan}

\affiliation{Universit\'e Paris Diderot\thanksmark{m4}, Ecole polytechnique\thanksmark{m1}, Universit\'e d'Evry\thanksmark{m2} 
\and University of Wroclaw\thanksmark{m3}}

\address{
Alain Virouleau \\
CMAP, Ecole Polytechnique \\
Route de Saclay \\
91128 Palaiseau cedex, France \\
\printead{e1}}

\address{
Agathe Guilloux \\
LAMME, Universit\'e d'Evry, CNRS \\
Universit\'e Paris-Saclay \\
91025, Evry, France \\
\printead{e2}}

\address{
St\'ephane Ga\"iffas \\
LPMA - Univ. Paris Diderot \\
Bâtiment Sophie Germain \\
Case courrier 7012  \\
75205 PARIS CEDEX 13 \\
FRANCE
\printead{e3}}

\address{
Malgorzata Bogdan \\
Department of Mathematics \\
University of Wroclaw \\
Wroclaw, Poland\\
\printead{e4}}
\end{aug}

\begin{abstract}
The problems of outliers detection and robust regression in a high-dimensional setting are fundamental in statistics, and have numerous applications.
Following a recent set of works providing methods for simultaneous robust regression and outliers detection,
we consider in this paper a model of linear regression with individual intercepts, in a high-dimensional setting.
We introduce a new procedure for simultaneous estimation of the linear regression coefficients and intercepts, using two dedicated sorted-$\ell_1$ penalizations, also called SLOPE~\cite{slope}.
We develop a complete theory for this problem: first, we provide sharp upper bounds on the statistical estimation error of both the vector of individual intercepts and regression coefficients.
Second, we give an asymptotic control on the False Discovery Rate (FDR) and statistical power for support selection of the individual intercepts.
As a consequence, this paper is the first to introduce a procedure with guaranteed FDR and statistical power control for outliers detection under the mean-shift model.
Numerical illustrations, with a comparison to recent alternative approaches, are provided on both simulated and several real-world datasets.
Experiments are conducted using an open-source software written in Python and C++.
\end{abstract}



\end{frontmatter}

\section{Introduction}
\label{sec:introduction}

Outliers are a fundamental problem in statistical data analysis.
Roughly speaking, an outlier is an observation point that differs from the data's ``global picture''~\cite{hawkins1980identification}.
A rule of thumb is that a typical dataset may contain between 1\% and 10\% of outliers~\cite{hampel2011robust}, or much more than that in specific applications such as web data, because of the inherent complex nature and highly uncertain pattern of users' web browsing~\cite{Gupta2016}.
This outliers problem was already considered in the early 50's~\cite{dixon1950analysis, grubbs1969procedures} and it motivated in the 70's the development of a new field called robust statistics~\cite{huber19721972, huber1981wiley}.

In this paper, we consider the problem of linear regression in the presence of outliers.
In this setting, classical estimators, such as the least-squares, are known to fail~\cite{huber19721972}.
In order to conduct regression analysis in the presence of outliers, roughly two approaches are well-known.
The first is based on detection and removal of the outliers to fit least-squares on ``clean'' data~\cite{weisberg2005applied}.
Popular methods rely on leave-one-out methods (sometimes called case-deletion), first described in~\cite{cook} with the use of residuals in linear regression. 
The main issue about these methods is that they are theoretically well-designed for the situations where only one given observation is an outlier.
Repeating the process across all locations can lead to well-known masking and swamping effects~\cite{masking}.
An interesting recent method that does not rely on a leave-one-out technique is the so-called IPOD~\cite{ipod}, a penalized least squares method with the choice of tuning parameter relying on a BIC criterion.
A second approach is based on robust regression, that considers loss functions that are less sensitive to outliers~\cite{huber1981wiley}. 
This relies on the $M$-estimation framework, that leads to good estimators of regression coefficients in the presence of outliers, thanks to the introduction of robust losses replacing the least-squares.
However, the computation of $M$-estimates is substantially more involved than that of the least-squares estimates, which to some extend counter-balance the apparent computational gain over previous methods.
Moreover, robust regression focuses only on the estimation of the regression coefficients, and does not allow directly to localize the outliers, see also for instance~\cite{robust} for a recent review.

Alternative approaches have been proposed to perform simultaneously outliers detection and robust regression.
Such methods involve median of squares~\cite{siegel1982robust}, S-estimation~\cite{rousseeuw1984robust} and more recently robust weighted least-squares~\cite{gervini2002class}, among many others, see also~\cite{Hadi} for a recent review on such methods.
The development of robust methods intersected with the development of sparse inference techniques recently.
Such inference techniques, in particular applied to high-dimensional linear regression, are of importance in statistics, and have been an area of major developments over the past two decades, with deep results in the field of compressed sensing, and more generally convex relaxation techniques~\cite{tibshirani1996regression, candes2006robust, candes2006stable, chen2001atomic,chandrasekaran2012convex}.
These led to powerful inference algorithms working under a sparsity assumption, thanks to fast and scalable convex optimization algorithms~\cite{bach2012optimization}.
The most popular method allowing to deal with sparsity and variable selection is the LASSO~\cite{lasso}, which is $\ell_1$-penalized least-squares, with improvements such as the Adaptive LASSO~\cite{alasso}, among a large set of other sparsity-inducing penalizations~\cite{buhlmann2011statistics,bach2012structured}.

Within the past few years, a large amount of theoretical results have been established to understand regularization methods for the sparse linear regression model, using so-called oracle inequalities for the prediction and estimation errors~\cite{BRT, CandesPlan,koltchinskii2008saint}, see also~\cite{buhlmann2011statistics,giraud} for nice surveys on this topic.
Another line of works focuses on variable selection, trying to recover the support of the regression coefficients with a high probability~\cite{Wainwright,CandesPlan,chretien}. 
Other types of loss functions~\cite{LAD} or penalizations~\cite{SCAD,slope} have also been considered. 
Very recently, the sorted-$\ell_1$ norm penalization has been 
introduced~\cite{slope,slope1,slopeminimax} and very strong statistical properties have been shown.
In particular, when covariates are orthogonal, SLOPE allows to recover the support of the regression coefficients with a control on the False Discovery Rate~\cite{slope}. 
For i.i.d covariates with a multivariate Gaussian distribution, oracle inequalities with optimal minimax rates have been shown, together with a control on a quantity which is very close to the FDR~\cite{slopeminimax}. 
For more general covariate distributions, oracle inequalities with an optimal convergence rate are obtained in~\cite{Tsyb}.

However, many high-dimensional datasets, with hundreds or thousands of covariates, do suffer from the presence of outliers. 
Robust regression and detection of outliers in a high-dimensional setting is therefore important. 
Increased dimensionality and complexity of the data may amplify the chances of an observation being an outlier, and this can have a strong negative impact on the statistical analysis. 
In such settings, many of the aforementioned outlier detection methods do not work well.
A new technique for outliers detection in a high-dimensional setting is proposed in~\cite{aggarwal2001outlier}, which tries to find the outliers by studying the behavior of projections from the data set.
A small set of other attempts to deal with this problem have been proposed in literature~\cite{loo,ro2015outlier,elasso,ipod,pwls}, and are described below with more details.

\section{Contributions of the paper}

Our focus is on possibly high dimensional linear regression where observations can be contaminated by gross errors. 
This so-called mean-shifted outliers model can be described as follows:
\begin{equation}
y_i = x_i^\top \beta^\star + \mu_i^\star + \varepsilon_i
\label{model}
\end{equation}
for $i=1, \ldots, n$, where $n$ is the sample size.
A non-zero $\mu_i^\star$ means that observation $i$ is an outlier, and $\beta^\star \in \R^p$, $x_i \in \R^p$, $y_i \in \R$ and $\varepsilon_i \in \R$ respectively stand for the linear regression coefficients, vector of covariates, label and noise of sample $i$.
For the sake of simplicity we assume throughout the paper that the noise is i.i.d centered Gaussian with known variance $\sigma^2$.

\subsection{Related works} 
\label{sub:related_works}

We already said much about the low-dimensional problem so we focus in this part on the high-dimensional one. 
The leave-one-out technique has been extended in~\cite{loo} to high-dimension and general regression cases, but the masking and swamping problems remains.
In other models, outliers detection in high-dimension also includes distance-based approaches~\cite{ro2015outlier} where the idea is to find the center of the data and then apply some thresholding rule.
The model~\eqref{model} considered here has been recently studied with LASSO penalizations~\cite{elasso} and 
hard-thresholding~\cite{ipod}.
LASSO was used also in~\cite{pwls}, but here outliers are modeled in the variance of the noise.
In~\cite{elasso, ipod}, that are closer to our approach, the penalization is applied differently: in~\cite{elasso}, the procedure named Extended-LASSO uses two different $\ell_1$ penalties for $\beta$ and $\mu$, with tuning parameters that are fixed according to theoretical results, while the IPOD procedure from~\cite{ipod} applies the same penalization to both vectors, with a regularization parameter tuned with a modified BIC criterion. 
In~\cite{elasso}, error bounds and a signed support recovery result are obtained for both the regression and intercepts coefficients. 
However, these results require that the magnitude of the coefficients is very large, which is one of the issues that we want to overcome with this paper.

It is worth mentioning that model~\eqref{model} can be written in a concatenated form $y=Z\gamma^\star + \varepsilon$, with $Z$ being the concatenation of the covariates matrix $X$ (with lines given by the $x_i$'s) and the identity matrix $I_n$ in $\R^n$, and $\gamma^\star$ being the concatenation of $\beta^\star$ and $\mu^\star$.
This leads to a regression problem with a very high dimension $n + p$ for the vector $\gamma^\star$.
Working with this formulation, and trying to estimate $\gamma^\star$ directly is actually a bad idea. This point  is illustrated experimentally in~\cite{elasso}, where it is shown that applying two different LASSO penalizations on $\beta$ and $\mu$ leads to a procedure that outperforms the LASSO on the concatenated vector. The separate penalization is even more important in case of SLOPE, whose aim is FDR control for the support recovery of $\mu^\star$. Using SLOPE directly on $\gamma^\star$ would mix the entries of $\mu$ and $\beta$ together, which would make FDR control practically impossible due to the correlations between covariates in the $X$ matrix.

\subsection{Main contributions} 

Given a vector $\lambda = [\lambda_1 \cdots \lambda_m] \in \R_+^m$ with non-negative and non-increasing entries, we define the sorted-$\ell_1$ norm of a vector $x \in \R^m$ as
\begin{equation}
\label{sl1}
\forall x \in \R^m, \; J_\lambda (x) = \sum_{j=1}^m \lambda_j \abs{x}_{(j)},
\end{equation}
where $\abs{x}_{(1)} \geq \abs{x}_{(2)} \geq \dots \geq \abs{x}_{(m)}$. In \cite{slope} and \cite{slope1} the sorted-$\ell_1$ norm was used as a penalty in the Sorted L-One Penalized Estimator (SLOPE) of coefficients in the multiple regression.
Degenerate cases of SLOPE are $\ell_1$-penalization whenever $\lambda_j$ are all equal to a positive constant, and null-penalization if this constant is zero.
We apply two different SLOPE penalizations on $\beta$ and~$\mu$, by considering the following optimization problem:
\begin{equation}
\label{pen}
\min_{\beta \in \R^p, \mu \in \R^n} \bigg\{\Vert y - X \beta - \mu \Vert_2 ^2 + 2\rho_1 J_{\tilde{\lambda}}(\beta)  + 2\rho_2 J_\lambda (\mu) \bigg\}
\end{equation}
where $\rho_1$ and $\rho_2$ are positive parameters,  $X$ is the $n \times p$ covariates matrix with rows $x_1, \ldots, x_n$,  $y = [y_1 \cdots y_n]^T$, $\mu = [\mu_1 \cdots \mu_n]^T$,  $\norm{u}_2$ is the Euclidean norm of a vector $u$ and  $\lambda = [\lambda_1 \cdots \lambda_p]$ and $\tilde \lambda = [\tilde \lambda_1 \cdots \tilde \lambda_n]$ are two vectors with non-increasing and non-negative entries.

In this artice we provide the set of sequences $\lambda$ and $\tilde \lambda$ which allow to obtain better error bounds for estimation of $\mu^\star$ and $\beta^\star$ than previously 
known ones~\cite{elasso}, see Section~\ref{SEC:BOUNDS} below.
Moreover, in Section~\ref{sec:fdr} we  provide specific sequences which, under some asymptotic regime, lead to a control of the FDR for the support selection of~$\mu^\star$, and such that the power of the procedure (\ref{pen}) converges to one.
Procedure~\eqref{pen} is therefore, to the best of our knowledge, the first proposed in literature to \emph{robustly estimate $\beta^\star$, estimate and detect outliers at the same time, with a control on the FDR for the multi-test problem of support selection of $\mu^\star$, and power consistency.}

We compare in Section~\ref{sec:simu} our procedure to the recent alternatives for this problem, that is the IPOD procedure~\cite{ipod} and the Extended-Lasso~\cite{elasso}. 
The numerical experiments given in Section~\ref{sec:simu} confirm the theoretical findings from Sections~\ref{SEC:BOUNDS} and~\ref{sec:fdr}.
As shown in our numerical experiments, the other procedures fail to guarantee FDR control or exhibit a lack of power when outliers are difficult to detect, namely when their magnitude is not far enough from the noise-level. 
It is particularly noticeable that our procedure overcomes this issue. 

The theoretical results proposed in this paper are based on two popular assumptions in compressed sensing or other sparsity problems, similar to the ones from~\cite{elasso}: first, a Restricted Eigenvalues (RE) condition~\cite{BRT} on $X$, then a mutual incoherence assumption \cite{DonohoHuo} between $X$ and $I_n$, which is natural since is excludes settings where the column spaces of $X$ and $I_n$ are impossible to distinguish.
Proofs of results stated in Sections~\ref{SEC:BOUNDS} and~\ref{sec:fdr} are given in Section~\ref{Appen:p1} and~\ref{Appen:p2}, while preliminary results are given in Sections~\ref{appen:prel1} and~\ref{appen:prel2}.
Section~\ref{appensim} provides contains supplementary extra numerical results.

\section{Upper bounds for the estimation of $\beta^\star$ and $\mu^\star$}
\label{SEC:BOUNDS}

Throughout the paper, $n$ is the sample size whereas $p$ is the number of covariables, so that $X\in\R^{n\times p}$. 
For any vector $u$, $\abs{u}_0$, $\norm{u}_1$ and $\norm{u}_2$ denote respectively the number of non-zero 
coordinates of $u$ (also called sparsity), the $\ell_1$-norm and the Euclidean norm.
We denote respectively by $\lambda_{\min}(A)$ and $\lambda_{\max}(A)$ the smallest and largest eigenvalue of a symmetric matrix $A$.
We work under the following assumption
\begin{ass}
\label{ass:sparsity_and_normalized_features}
We assume the following sparsity assumption\textup:
\begin{equation}
	\abs{\beta^\star}_0 \leq k \quad \text{ and } \quad \abs{\mu^\star}_0 \leq s
\end{equation}
for some positive integers $k$ and $s$,
and we assume that the columns of $X$ are normalized, namely $\norm{Xe_i}_2 = 1$ for $i=1, \ldots, n$, 
where $e_i$ stands for the $i$-th element of the canonical basis.
\end{ass}
For the results of this Section, we consider procedure~\eqref{pen} with the following choice of $\lambda$:
\begin{equation}
	\label{eq:slope_weights}
	\lambda_i = \sigma \sqrt{\log \Big(\frac{2n}{i}\Big)},
\end{equation}
for $i=1,\dots,n$, and we consider three possibilities for $\tilde{\lambda}$, corresponding to no 
penalization, $\ell_1$ penalization and SLOPE penalization on $\beta$.

Table~\ref{tab:rates} below gives a quick view of the convergence rates of the squared $\ell_2$ estimation errors of $\beta^\star$ and $\mu^\star$ obtained in Theorems~\ref{thm:upper_bound_nobeta},~\ref{thm:upper_bound_beta_l1} and~\ref{thm:upper_bound_beta_sl1}.
We give also the convergence rate obtained in~\cite{elasso} for $\ell_1$ penalization applied to $\beta$ and $\mu$.
In particular, we see that using two SLOPE penalizations leads to a better convergence rate than the use of $\ell_1$ penalizations.
\begin{table}[htbp]
\centering
\caption{Convergence rates, up to constants, associated to several penalization techniques. 
NO means no-penalization, L1 stands for $\ell_1$ penalization, while SL1 stands for SLOPE.
We observe that SL1 + SL1 leads to a better convergence rate than L1 + L1.}
\label{tab:rates}
\begin{tabular}{@{}llll@{}}
\hline
\multicolumn{1}{l}{\begin{tabular}[l]{@{}l@{}}Penalization \\ 
($\beta$ / $\mu$)\end{tabular}} & \multicolumn{1}{l}{Convergence rates} & Reference \\
\hline
\addlinespace[0.5em] 
NO/SL1  & $p\vee s \log(n/s)$ & Theorem~\ref{thm:upper_bound_nobeta} \\ 
\addlinespace[0.5em]
L1/L1 & $k\log p \vee s\log n$ & \cite{elasso} \\
\addlinespace[0.5em]
L1/SL1 & $k\log p \vee s\log(n/s)$ & Theorem~\ref{thm:upper_bound_beta_l1} \\
\addlinespace[0.5em]
SL1/SL1 & $k\log (p/k) \vee s\log (n/s)$ & Theorem~\ref{thm:upper_bound_beta_sl1} \\
\addlinespace[0.5em]
\hline
\end{tabular}
\end{table}
Condition~\ref{Hyp} below is a Restricted Eigenvalue (RE) type of condition which is adapted to our problem.
Such an assumption is known to be mandatory in order to derive fast rates of convergence for penalizations based on the convex-relaxation 
principle~\cite{zhang2014lower}.
\begin{condition}
	\label{Hyp}
	Consider two vectors $\lambda = (\lambda_i)_{i=1,\dots,n}$ and $\tilde \lambda = (\tilde{\lambda}_i)_{i=1,\dots,p}$ with non-increasing and positive entries, and consider positive integers $k, s$ and $c_0 > 0$.
	We define the cone $\mathcal{C}(k,s,c_0)$ of all vectors $[\beta^\top, \mu^\top]^\top \in \R^{p + n}$ satisfying
	\begin{equation}
		\label{Cewre}
		\sum_{j=1}^p \frac{\tilde{\lambda}_j}{\tilde{\lambda}_p}\abs{\beta}_{(j)} 
		+ \sum_{j=1}^n \frac{\lambda_j}{\lambda_n}\abs{\mu}_{(j)} \leq (1+c_0) \big( \sqrt{k}\norm{\beta}_2 + \sqrt{s}\norm{\mu}_2 \big).
	\end{equation}
	We also define the cone $\mathcal{C}^p (s,c_0)$ of all vectors $[\beta^\top, \mu^\top]^\top \in \R^{p + n}$ satisfying
	\begin{equation}
		\label{Cewre_nobeta}
	\sum_{j=1}^n \frac{\lambda_j}{\lambda_n}\abs{\mu}_{(j)} \leq (1+c_0) \big( \sqrt{p}\norm{\beta}_2 + \sqrt{s}\norm{\mu}_2 \big).
	\end{equation}
	We assume that there are constants $\kappa_1, \kappa_2 > 0$ with $\kappa_1 > 2\kappa_2$ such that $X$ satisfies the following, either for all
	 $[\beta^\top, \mu^\top]^\top  \in \mathcal{C}(k,s,c_0)$ or for all
	  $[\beta^\top, \mu^\top]^\top  \in \mathcal{C}^p(s,c_0)$\textup:
	\begin{align}
		\label{eq:hyp_eq_1}
		\norm{X\beta}_2^2 + \norm{\mu}_2^2 &\geq \kappa_1 \big(\norm{\beta}_2^2 + \norm{\mu}_2^2\big) \\
		\label{eq:hyp_eq_2}
		\vert\langle X\beta, \mu \rangle\vert &\leq \kappa_2\big(\norm{\beta}_2^2 + \norm{\mu}_2^2\big).
	\end{align}
\end{condition}
Equation~\eqref{Cewre_nobeta} corresponds to the particular case where 
we do not penalize the regression coefficient $\beta$, namely $\tilde \lambda_i=0$ for all $i$. 
Note also that Condition~\ref{Hyp} entails
\begin{equation*}
	\norm{X \beta + \mu}_2 \geq \sqrt{\kappa_1 - 2 \kappa_2} \sqrt{\norm{\beta}_2^2 + \norm{\mu}_2^2},
\end{equation*}
which actually corresponds to a RE condition on $[X^\top I_n]^\top$ and that Equation~\eqref{eq:hyp_eq_1} is satisfied if $X$ satisfies a RE condition with constant $\kappa < 1$. 
Finally, note that Equation~\eqref{eq:hyp_eq_2}, called mutual incoherence in the literature of compressed sensing, requires in this context that for all $\beta$ and $\mu$ from the respective cones the potential regression predictor $X \beta$ is sufficiently not-aligned with potential outliers $\mu$. 
An extreme case of violation of this assumption occurs when $X = I_n$, where we cannot separate the regression coefficients from the outliers.

The Condition~\ref{Hyp} is rather mild and e.g. for a wide range of random designs. Specifically, Theorem~\ref{thm:RE} below, shows that it holds with large probability whenever $X$ has i.i.d $\mathcal N(0,\Sigma)$ rows, with $\lambda_{min}(\Sigma)>0$, and the vectors $\beta$ and $\mu$ are sufficiently sparse.

\begin{thm}
\label{thm:RE}
	Let $X' \in \R^{n\times p}$ be a random matrix with i.i.d $\mathcal{N}(0, \Sigma)$ rows and $\lambda_{min}(\Sigma)>0$.
	Let $X$ be the corresponding matrix with normalized columns. 
	Given positive integers $k, s$ and $c_0 > 0$, define $r = s $ $\vee$  $k (1+c_0)^2$. 
	If 
	\begin{equation*}
		\sqrt{n} \geq C\sqrt{r} \quad \text{ and } \quad \sqrt{n}\geq C'\sqrt{r\log (p \vee n)}
	\end{equation*}
	with
	\begin{equation*}
		C \geq 30 \frac{\sqrt{\lambda_{\max}(\Sigma)}}{\min_j \Sigma_{jj}}\Big(\frac{256\times 5 \max_j \Sigma_{jj}}{\lambda_{\min}(\Sigma)} \vee 16\Big)
		\quad \text{ and } \quad C' \geq 72\sqrt{10}\frac{(\max_j \Sigma_{jj})^{3/2}}{\min_j \Sigma_{jj}\sqrt{\lambda_{\text{min}}(\Sigma)}},
	\end{equation*}
	then there are $c, c'> 0$ such that for any $[\beta^\top, \mu^\top]^\top  \in \mathcal{C}(k,s,c_0)$, we have
	\begin{align*}
		\norm{X\beta}_2^2 + \norm{\mu}_2^2 &\geq \min\Big\{\frac{\lambda_{\min}(\Sigma)}{128\time 5 (\max_j \Sigma_{jj})^2}, 
		\frac{1}{8}\Big\} \big(\norm{\beta}_2^2 + \norm{\mu}_2^2\big) \\
		2\vert\langle X\beta, \mu \rangle\vert &\leq \min\Big\{\frac{\lambda_{\min}(\Sigma)}{256\times 5 (\max_j \Sigma_{jj})^2}, \frac{1}{16}\Big\} \big(\norm{\beta}_2^2 + \norm{\mu}_2^2\big) 
	\end{align*}
	with a probability greater than $1 - cn\exp(-c'n)$. 
	These inequalities also hold for any $[\beta^\top, \mu^\top]^\top  \in \mathcal{C}^p(s,c_0)$ when $k$ is replaced by $p$ in the above conditions.
\end{thm}

The proof of Theorem~\ref{thm:RE} is given in Appendix~\ref{Appen:proof_RE}. It is based on recent bounds results for Gaussian random matrices \cite{REgauss}.
The numerical constants in Theorem~\ref{thm:RE} are far from optimal and chosen for simplicity so that $\kappa_1 > 2 \kappa_2$ as required in Assumption~\ref{Hyp}.
A typical example for $\Sigma$ is the Toeplitz matrix $[a^{\abs{i-j}}]_{i,j}$ with $a \in [0,1)$, for which
 $\lambda_{\min}(\Sigma)$ is equal to $1-a$ \cite{REgauss}.
The required lower bound on $n$ is non-restrictive, since $k$ and $s$ correspond to the sparsity of
 $\beta^\star$ and $\mu^\star$, that are typically much smaller than $n$. Note also that $\mathcal{C}^p(s,c_0)$ will only be used in low dimension, and in this case $p$ is again much smaller than $n$.

Let us define $\kappa = \sqrt{\kappa_1 - 2\kappa_2}$ for the whole Section, 
with $\kappa_1$ and $\kappa_2$ defined 
in Assumption~\ref{Hyp}.
The three theorem below and their proof are very similar in nature, but differ in some details, therefore are stated and proved separately. We emphasize that the proof give slightly more general versions of the theorems, allowing the same result with $\hat{\mu}$ having any given support containing $\mathrm{Supp}(\mu^\star)$. This is of great theoretical interest and is a key point for the support detection of $\mu^\star$ investigated in \ref{sec:fdr}. The proof use a very recent bound on the inner product between a white Gaussian noise and any vector, involving the sorted $\ell_1$ norm \cite{Tsyb}.
Our first result deals with linear regression with outliers and no sparsity assumption on $\beta^\star$.
We consider procedure~\eqref{pen} with no penalization on $\beta$, namely
\begin{equation}
\label{eq:procedure_no_beta}
	(\hat \beta, \hat \mu) \in \argmin_{\beta \in \R^p, \mu \in \R^n} 
	\Big\{ \Vert y - X \beta - \mu \Vert_2 ^2 + 2\rho J_\lambda (\mu) \Big\},
\end{equation}
with $J_\lambda$ given by~\eqref{sl1} and weights $\lambda$ given by~\eqref{eq:slope_weights}, and with $\rho\geq 2(4+\sqrt{2})$.
Theorem~\ref{thm:upper_bound_nobeta}, below, shows that a convergence rate for procedure~\eqref{eq:procedure_no_beta} is indeed $p\vee s\log(n/s)$, as reported in Table~\ref{tab:rates} above.

\begin{thm}
\label{thm:upper_bound_nobeta}
Suppose that Assumption~\ref{ass:sparsity_and_normalized_features} is met with $k = p$, and that $X$ 
satisfies Assumption~\ref{Hyp} on the cone $\mathcal C(k_1,s_1,4)$ 
with $k_1 = p/\log 2 $ and $s_1 =s\log\left(2en/s\right)/\log 2$.
Then, the estimators $(\hat \beta, \hat \mu)$ given by~\eqref{eq:procedure_no_beta} satisfy
\begin{equation*}
\Vert\hat{\beta} - \beta^\star \Vert_2^2 + \norm{\hat{\mu}-\mu^\star}_2^2 \leq 
\frac{4\rho^2}{\kappa^4}\sum_{j=1}^s \lambda_j^2+ \frac{5\sigma^2 }{\kappa^4} p
\leq \frac{\sigma^2}{\kappa^4}\left(4\rho^2 s \log\left(\frac{2en}{s}\right)+5p\right),
\end{equation*}
with a probability larger than $1-\left(s/2n\right)^{s} / 2 - e^{-p}$.
\end{thm}
The proof of Theorem~\ref{thm:upper_bound_nobeta} is given in Appendix~\ref{proof:upper_bound_nobeta}. 
The second result involves a sparsity assumption on $\beta^\star$ and 
considers $\ell_1$ penalization for $\beta$.
We consider this time
\begin{equation}
\label{eq:2}
(\hat \beta, \hat \mu) \in \argmin_{\beta, \mu} \Big\{ \Vert y - X\beta - \mu \Vert_2 ^2 + 
2 \nu \norm{\beta}_1 + 2\rho J_\lambda (\mu) \Big\},
\end{equation}
where $\nu = 4\sigma\sqrt{\log p}$ is the regularization level for $\ell_1$ penalization, $\rho \geq 2(4+\sqrt{2})$ and $J_\lambda$ is given by~\eqref{sl1}. Theorem \ref{thm:upper_bound_beta_l1}, below, shows that a convergence rate for procedure~\eqref{eq:2} is indeed $k \log p\vee s\log(n/s)$, as reported in Table~\ref{tab:rates} above.

\begin{thm}
Suppose that Assumption~\ref{ass:sparsity_and_normalized_features} is met and that $X$ 
satisfies Assumption~\ref{Hyp} on the cone $\mathcal C(k_1,s_1,4)$ with
 $k_1 = 16k\log p/\log 2 $ and $s_1 = s\log (2en/s)/\log 2$. Suppose also that $\sqrt{\log p}\geq \rho\log 2 /4$.
Then, the estimators $(\hat \beta, \hat \mu)$ given by~\eqref{eq:2} satisfy
\begin{equation*}
	\Vert\hat{\beta}-\beta^\ast\Vert_2^2 + \norm{\hat{\mu}-\mu^\ast}_2^2 \leq \frac{36}{\kappa^4}\sigma^2 k\log p + \frac{4\rho^2}{\kappa^4}\sum_{j=1}^s \lambda_j^2 \leq \frac{4 \sigma^2}{\kappa^4}\left(9k\log p+\rho^2 s \log \left(\frac{2en}{s}\right)\right),
\end{equation*}
with a probability larger than $1-\left(s/2n\right)^{s} / 2 - 1/p$.
\label{thm:upper_bound_beta_l1}
\end{thm}
The proof of Theorem~\ref{thm:upper_bound_beta_l1} is given in Appendix~\ref{proof:beta_sl1}.
The third result is obtained using SLOPE both on $\beta$ and $\mu$, namely
\begin{equation}
\label{eq:procedure_with_two_slopes}
	(\hat \beta, \hat \mu) \in \argmin_{\beta, \mu} \Big\{ \Vert y - X \beta - \mu \Vert_2 ^2 
	+ 2\rho J_{\tilde{\lambda}}(\beta)  + 2\rho J_\lambda (\mu) \Big\}
\end{equation}
where $\rho \geq 2(4+\sqrt{2})$, $J_\lambda$ is given by~\eqref{sl1}, and where
\begin{equation*}
\tilde{\lambda}_j = \sigma\sqrt{\log \Big(\frac{2p}{j} \Big)}
\end{equation*}
for $j=1,\dots, p$.
Theorem \ref{thm:upper_bound_beta_sl1}, below, shows that the rate of convergence of estimators provided by (\ref{eq:procedure_with_two_slopes}) is indeed $k\log(p/k)\vee s\log(n/s)$, as presented in Table~\ref{tab:rates}.
\begin{thm}
Suppose that Assumption~\ref{ass:sparsity_and_normalized_features} is met and that $X$ 
satisfies Assumption~\ref{Hyp} on the cone $\mathcal C(k_1,s_1,4)$ with
$k_1 = k \log (2ep / k) / \log 2$ and $s_1 =s \log(2en / s) / \log 2$.
Then, the estimators $(\hat \beta, \hat \mu)$ given by~\eqref{eq:procedure_with_two_slopes} satisfy
\begin{eqnarray}
	\Vert\hat{\beta}-\beta^\ast\Vert_2^2 + \norm{\hat{\mu}-\mu^\ast}_2^2 &\leq&\frac{C'}{\kappa^4}
	\Big( \sum_{j=1}^k \tilde{\lambda}_j^2  + \sum_{j=1}^s \lambda_j^2  \Big) \nonumber \\
	&\leq& \frac{C'\sigma^2}{\kappa^4} \left(k \log\left(\frac{2ep}{k}\right)+s\log \left(\frac{2en}{s}\right)\right),
\end{eqnarray}
with a probability greater than $1 - (s / 2n)^{s} / 2 - (k/2p)^{k} / 2$,
where $C' = 4\rho^2 \vee (3+C)^2 / 2$.
\label{thm:upper_bound_beta_sl1}
\end{thm}
The proof of Theorem~\ref{thm:upper_bound_beta_sl1} is given in Appendix~\ref{proof:beta_sl1}.
Note that according to Theorem~\ref{thm:RE}, the assumptions of Theorem~\ref{thm:upper_bound_beta_sl1} are satisfied with probability converging to one when the rows of $X$ are i.i.d from the multivariate Gaussian distribution with the positive definite covariance matrix, and when the signal is sparse such that $(k \vee s) \log(n \vee p)=o(n)$.

\section{Asymptotic FDR control and power for the selection of the support of $\mu^\star$}
\label{sec:fdr}

We consider the multi-test problem with null-hypotheses
\begin{equation*}
	H_i \;: \; \mu^\ast_i = 0
\end{equation*}
for $i=1, \ldots, n$, and we consider the multi-test that rejects $H_i$ whenever
 $\hat{\mu}_i \neq 0$, where $\hat \mu$ (and $\hat \beta$) are given either 
 by~\eqref{eq:procedure_no_beta}, \eqref{eq:2} or~\eqref{eq:procedure_with_two_slopes}.
When $H_i$ is rejected, or ``discovered'', we consider that sample $i$ is an outlier.
Note however that in this case, the value of $\hat \mu_i$ gives extra information on how much 
sample $i$ is oulying.

We use the FDR as a standard Type~I error for this multi-test problem~\cite{benjamini1995controlling}.
The FDR is the expectation of the proportion of falses discoveries among all discoveries.
Letting $V$ (resp. $R$) be the number of false rejections (resp. the number of rejections), the FDR is defined as
\begin{equation}
\label{eq:fdr}
\mathrm{FDR}(\hat{\mu}) = \E \left[\frac{V}{R\vee 1}\right] = 
\E\left[\frac{\#\{ i: \mu_i^\star = 0, \hat{\mu}_i\neq 0\}}{\#\{i: \hat{\mu}_i\neq 0\}}\right].
\end{equation}
We use the Power to measure the Type~II error for this multi-test problem. The Power is the expectation of the proportion of true discoveries. It is defined as 
\begin{equation}
\label{eq:power}
\mathrm{\Pi}(\hat{\mu}) = \E\left[\frac{\#\{ i: \mu_i^\star \neq 0, \hat{\mu}_i\neq 0\}}{\#\{i: \mu_i^\star \neq 0\}}\right],
\end{equation}
the Type~II error is then given by $1-\Pi (\hat{\mu})$.
 
For the linear regression model without outliers, a multi-test for the support selection of $\beta^\star$ with controlled FDR based on SLOPE is given in~\cite{slope} and \cite{slope1}. 
Specifically, it is shown that SLOPE with weights
\begin{equation}
	\label{eq:slope_bh_weights}
	\lambda_i ^{\mathrm{BH}}= \sigma\Phi^{-1}\Big( 1 - \frac{iq}{2n} \Big)
\end{equation}
for $i=1,\dots, n$, where $\Phi$ is the cumulative distribution function of $\mathcal N(0, 1)$ and $q \in (0, 1)$, controls FDR at the level $q$ in the multiple regression problem with orthogonal design matrix $X^T X=I$. It is also observed that when the columns of $X$ are not orthogonal but independent the weights have to be substantially increased to guarantee FDR control. This effect results from the random correlations between columns of $X$ and the shrinkage of true nonzero coefficients, and in context of LASSO have been thoroughly discussed in \cite{Su_2017}.

In this paper we substantially extend current results on FDR controlling properties of SLOPE.
Specifically, Theorem~\ref{THM:FDR} below gives asymptotic controls of $\mathrm{FDR}(\hat{\mu})$ and $\mathrm{\Pi}(\hat{\mu})$ for the procedures~\eqref{eq:procedure_no_beta},~\eqref{eq:2} and~\eqref{eq:procedure_with_two_slopes}, namely different penalizations on $\beta$ and SLOPE applied on $\mu$, with slightly increased weights
\begin{equation}
	\label{eq:increased_lambda}
	\lambda = (1+\epsilon)\lambda^{\mathrm{BH}},	
\end{equation}
where $\epsilon > 0$, see also~\cite{slopeminimax}.
This choice of $\lambda$ also yields optimal convergence rates, however considering it in Section~\ref{SEC:BOUNDS} would lead to some extra technical difficulties.
Under appropriate assumptions on $p, n$, the signal sparsity and the magnitude of outliers, Theorem~\ref{THM:FDR}  not only gives  FDR control, but also proves that the Power is actually going to $1$.

Note that all asymptotics considered here are with respect to the sample size $n$, namely the statement $d \rightarrow +\infty$ means that $d = d_n \rightarrow +\infty$ with $n \rightarrow +\infty$.
\begin{thm}
\label{THM:FDR}
Suppose that there is a constant $M$ such that the entries of $X$ satisfy $\abs{x_{i,j}}\sqrt{n}\leq M$ 
for all $i,j\in\{1,\dots n \}$, and suppose that
\begin{equation*}
	|\mu_i^\star| \geq (1+\rho(1+2\epsilon))2\sigma\sqrt{\log n}
\end{equation*}
for any $i=1, \ldots, n$ such that $\mu_i^\star \neq 0$. 
Suppose also that $s \rightarrow +\infty$.
Then, consider $(\hat \beta, \hat \mu)$ given either by~\eqref{eq:procedure_no_beta},~\eqref{eq:2} and~\eqref{eq:procedure_with_two_slopes}, with $\lambda$ given by~\eqref{eq:increased_lambda}.
For Procedure~\eqref{eq:procedure_no_beta}, assume the same as in~Theorem~\ref{thm:upper_bound_nobeta}, and that 
\begin{equation*}
	\frac {p (s\log (n/s) \vee p)}{n} \rightarrow 0.
\end{equation*}
For Procedure~\eqref{eq:2}, assume the same as in~Theorem~\ref{thm:upper_bound_beta_l1}, and that 
\begin{equation*}
	\frac{(s\log (n/s) \vee k\log p )^2}{n} \rightarrow 0. 
\end{equation*}
For Procedure~\eqref{eq:procedure_with_two_slopes}, assume the same as in~Theorem~\ref{thm:upper_bound_beta_sl1}, and that 
\begin{equation*}
	\frac{\left(s\log (n/s) \vee k\log (p/k)\right)^2}{n} \rightarrow 0.
\end{equation*}
Then, the following properties hold:
\begin{equation}
\Pi (\hat{\mu}) \rightarrow 1, \qquad \limsup \mathrm{FDR}(\hat{\mu}) \leq q.
\end{equation}
\end{thm}

The proof of Theorem~\ref{THM:FDR} is given in Appendix~\ref{Appen:p2}. 
It relies on a careful look at the KKT conditions, also known as the dual-certificate method~\cite{Wainwright} or resolvent solution~\cite{slopeminimax}.
The assumptions of Theorem~\ref{THM:FDR} are natural. 
The boundedness assumption on the entries of $X$ are typically satisfied with a large probability when $X$ is Gaussian for instance.
When $n \rightarrow +\infty$, it is also natural to assume that $s \rightarrow +\infty$ (let us recall that $s$ stands for the sparsity of the sample outliers $\mu \in \R^n$).
The asymptotic assumptions roughly ask for the rates in Table~\ref{tab:rates} to converge to zero.
Finally, the assumption on the magnitude of the non-zero entries of $\mu^\ast$ is somehow unavoidable, since it allows 
to distinguish outliers from the Gaussian noise.
We emphasize that good numerical performances are actually obtained with lower magnitudes, as illustrated in Section~\ref{subsec:simu}.

\section{Numerical experiments}
\label{sec:simu}

In this section, we illustrate the performance of procedure~\eqref{eq:procedure_no_beta} and  procedure~\eqref{eq:procedure_with_two_slopes} both on simulated and real-word datasets, and compare them to several state-of-the art baselines described below. Experiments are done using the open-source \texttt{tick} library, available at \url{https://x-datainitiative.github.io/tick/}, notebooks allowing to reproduce our experiments are available on demand to the authors.

\subsection{Considered procedures} 
\label{par:baselines_procedures}

We consider the following baselines, featuring the best methods available in literature for the joint problem of outlier detection and estimation of the regression coefficients, together with the methods introduced in this paper.

\paragraph{E-SLOPE} 
\label{par:experiments_e_slope}

It is procedure~\eqref{eq:procedure_with_two_slopes}.
The weights used in both SLOPE penalizations are given by~\eqref{eq:slope_bh_weights}, with $q = 5\%$ (target 
FDR), except in low-dimensional setting where we do not apply any penalization on $\beta$. Similar results for $q=10\%$ and $q=20\%$ are provided in Appendix~\ref{appensim}.

\paragraph{E-LASSO} 
\label{par:experiments_lasso}

This is Extended LASSO from~\cite{elasso}, that uses two dedicated $\ell_1$-penalizations for $\beta$ and $\mu$ with respective tuning parameters $\lambda_{\beta} = 2\sigma \sqrt{\log p}$ and $\lambda_{\mu} = 2\sigma\sqrt{\log n}$.

\paragraph{IPOD} 
\label{par:experiments_IPOD}

This is (soft-)IPOD from~\cite{ipod}. 
It relies on a nice trick, based on a $QR$ decomposition of $X$. 
Indeed, write $X = QR$ and let $P$ be formed by $n-p$ column vectors completing the column 
vectors of $Q$ into an orthonormal basis, and introduce $\tilde{y} = P^\top y \in\R^{n-p}$.
Model~\eqref{model} can be rewritten as $\tilde{y} = P^\top \mu^\ast + \varepsilon$, a new high-dimensional
linear model, where only $\mu^\ast$ is to be estimated.
The IPOD then considers the Lasso procedure applied to this linear model, and a BIC criterion is used to choose the tuning parameter of $\ell_1$-penalization.
Note that this procedure, which involves a QR decomposition of $X$, makes sense only for $p$ significantly smaller than $n$, so that we do not report the performances of IPOD on simulations with a large $p$.

\paragraph{LassoCV} 
\label{par:experiments_LassoCV}

Same as IPOD but with tuning parameter for the penalization of individual intercepts chosen by cross-validation. 
As explained above, cross-validation is doomed to fail in the considered model, but results are shown for the sake of completeness.

\paragraph{SLOPE} 
\label{par:slope}

It is SLOPE applied to the concatenated problem, namely $y = Z \gamma^\star + \varepsilon$, where $Z$ is the concatenation of $X$ and $I_n$ and $\gamma^\star$ is the concatenation of $\beta^\star$ and $\mu^\star$.
We use a single SLOPE penalization on $\gamma$, with weights equal to~\eqref{eq:slope_bh_weights}.
We report the performances of this procedure only in high-dimensional experiments, since it always penalizes $\beta$.
This is considered mostly to illustrate the fact that working on the concatenated problem is indeed a bad idea, and that two distinct penalizations must be used on $\beta$ and $\mu$.

\medskip
Note that the difference between IPOD and E-LASSO is that, as explained in~\cite{elasso}, the weights used for 
E-LASSO to penalize $\mu$ (and $\beta$ in high-dimension) are fixed, while the weights in IPOD are data-dependent. 
Another difference is that IPOD do not extend well to a high-dimensional setting, since its natural extension
(considered in \cite{ipod}) is a thresholding rule on the the concatenated problem, which is, as explained before, and as illustrated in our numerical experiments, poorly performing.
Another problem is that there is no clear extension of the modified BIC criterion proposed in~\cite{ipod} for high-dimensional problems.

The tuning of the SLOPE or $\ell_1$ penalizations in the procedure described above require the knowledge of the noise level.
We overcome this simply by plugging in~\eqref{eq:slope_bh_weights} or wherever it is necessary a robust estimation of the variance: we first fit a Huber regression model, and apply a robust estimation of the variance of its residuals.
All procedures considered in our experiments use this same variance estimate.

\begin{rem}
The noise level can be estimated directly by the Huber regression since in our simulations $p<n$. When $p$ is comparable or larger than $n$ and the signal 
\textup(both $\beta^{\star}$ and $\mu^{\star}$\textup) is sufficiently sparse one can jointly estimate the noise level and other model parameters in the spirit of {\it scaled} LASSO \cite{sun2012scaled}. The corresponding iterative procedure for SLOPE was proposed and investigated  
in \cite{slope} in the context of high-dimensional regression with independent regressors. However, we feel that the problem of selection of the optimal estimator of $\sigma$ in ultra-high dimensional settings still requires a separate investigation and we postpone it for a further research.	
\end{rem}

\subsection{Simulation settings}
\label{subsec:simu}

The matrix $X$ is simulated as a matrix with i.i.d row distributed as $\mathcal N(0, \Sigma)$, with Toeplitz 
covariance $\Sigma_{i, j} = \rho^{|i-j|}$ for $1 \leq i, j \leq p$, with moderate correlation $\rho = 0.4$.
Some results with higher correlation $\rho = 0.8$ are given in Appendix~\ref{appensim}.
The columns of $X$ are normalized to $1$.
We simulate $n=5000$ observations according to model~\eqref{model} with $\sigma=1$ 
and $\beta^\ast_i = \sqrt{2\log p}$. Two levels of magnitude are considered for $\mu^\star$: \emph{low-magnitude}, where $\mu^\star_i = \sqrt{2\log n}$ and \emph{large-magnitude}, where $\mu^\star_i = 5\sqrt{2\log n}$.
In all reported results based on simulated datasets, the sparsity of $\mu^\star$ varies between $1\%$
to $50\%$, and we display the averages of FDR, MSE and power over 100 replications.

\paragraph{Setting~1 (low-dimension)}

This is the setting described above with $p=20$. Here $\beta^{\star}_1=\ldots=\beta^{\star}_{20}=\sqrt{2\log 20}$.  

\paragraph{Setting~2 (high-dimension)}

This is the setting described above with $p=1000$ and a sparse $\beta^\star$ with sparsity $k=50$, with non-zero entries chosen uniformly at random.

\subsection{Metrics} 
\label{sub:metrics}

In our experiments, we report the ``MSE coefficients'', namely  $\| \hat \beta - \beta^\star \|_2^2$ and the ``MSE intercepts'', namely $\| \hat \mu - \mu^\star \|_2^2$.
We report also the FDR~\eqref{eq:fdr} and the Power~\eqref{eq:power} to assess the procedures for the problem of outliers detection, where the expectations are approximated by averages over 100 simulations.

\subsection{Results and conclusions on simulated datasets} 
\label{sub:results_and_conclusions_on_simulated_datasets_}

We comment the displays provided in Figures~\ref{fig:lowdim1},~\ref{fig:lowdim2} and~\ref{fig:highdim1} below.
On Simulation Setting~2 we only display results for the low magnitude case, since it is the most challenging one. 
\begin{itemize}
\item In Figures~\ref{fig:lowdim1} and~\ref{fig:lowdim2}, LassoCV is very unstable, which is expected since cross-validation cannot work in the considered setting: data splits are highly non-stationary, because of the significant amount of outliers.
\item In the low dimensional setting, our procedure $E-SLOPE$ allows for almost perfect FDR control and its MSE is the smallest among all considered methods. Note that in this setting the MSE is plotted after debiasing the estimators, performing ordinary least squares on the selected support.
\item In the sparse (on $\beta$)  high dimensional setting with correlated regressors, E-SLOPE allows to keep FDR below the nominal level even when the outliers consist 50\% of the total data points. It also allows to maintain a small MSE and high power.
The only procedure that improves E-SLOPE in terms of MSE for $\mu$ is SLOPE in Figure~\ref{fig:highdim1}, at the cost of a worse FDR control.
\item E-SLOPE provides a massive gain of power compared to previous state-of-the-art procedures (power is increased by more than $30\%$) in settings where outliers are difficult to detect.
\end{itemize}
\begin{figure}
\centering
\includegraphics[width=\linewidth]{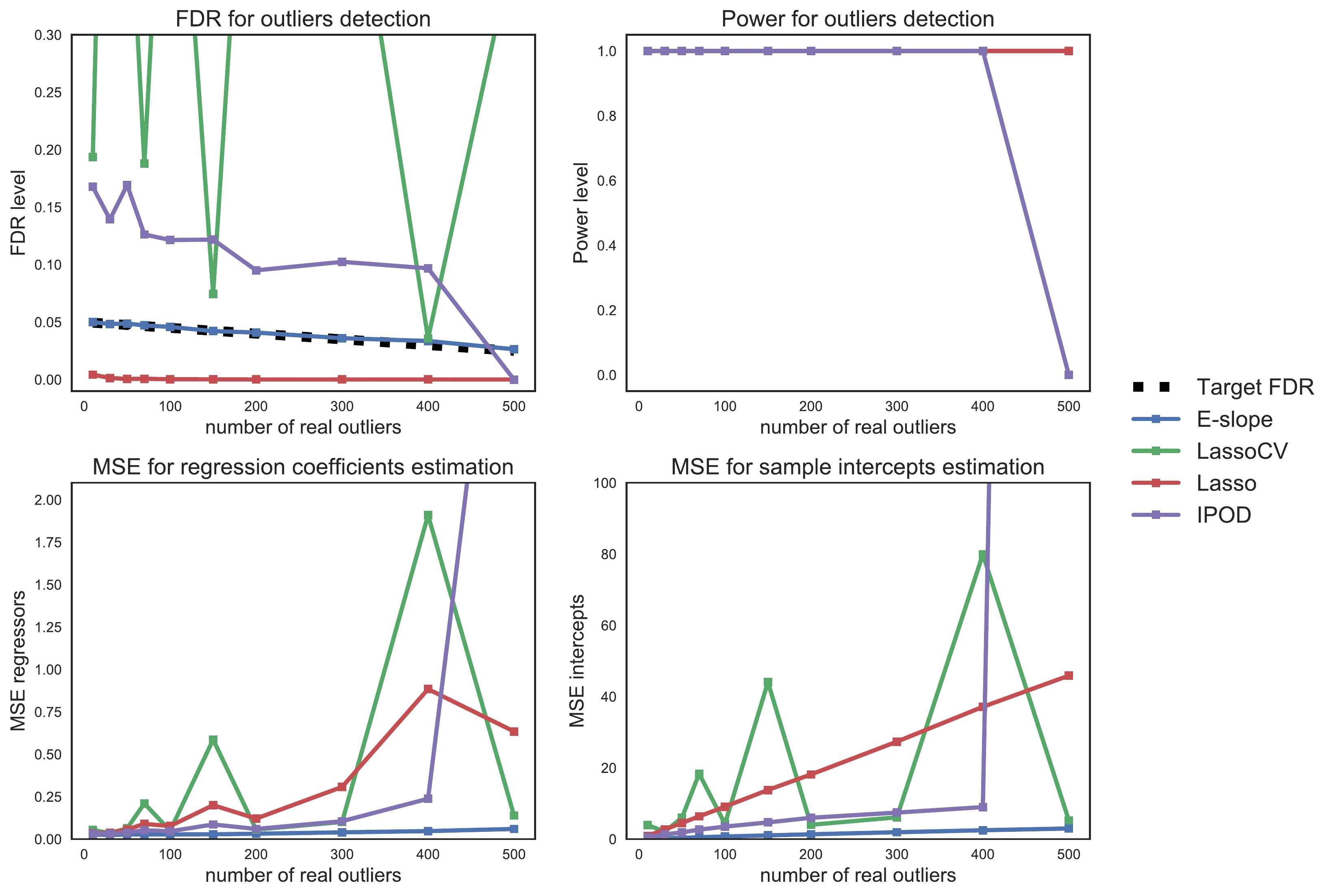}
\caption{Results for Simulation Setting~1 with high-magnitude outliers. First row gives the FDR (left) and power (right) of each considered procedure for outliers discoveries. 
Second row gives the MSE for regressors (left) and intercepts (right). 
E-SLOPE gives perfect power, is the only one to respect the required FDR, and provides the best MSEs.}
\label{fig:lowdim1}
\end{figure}
\begin{figure}
\centering
\includegraphics[width=\linewidth]{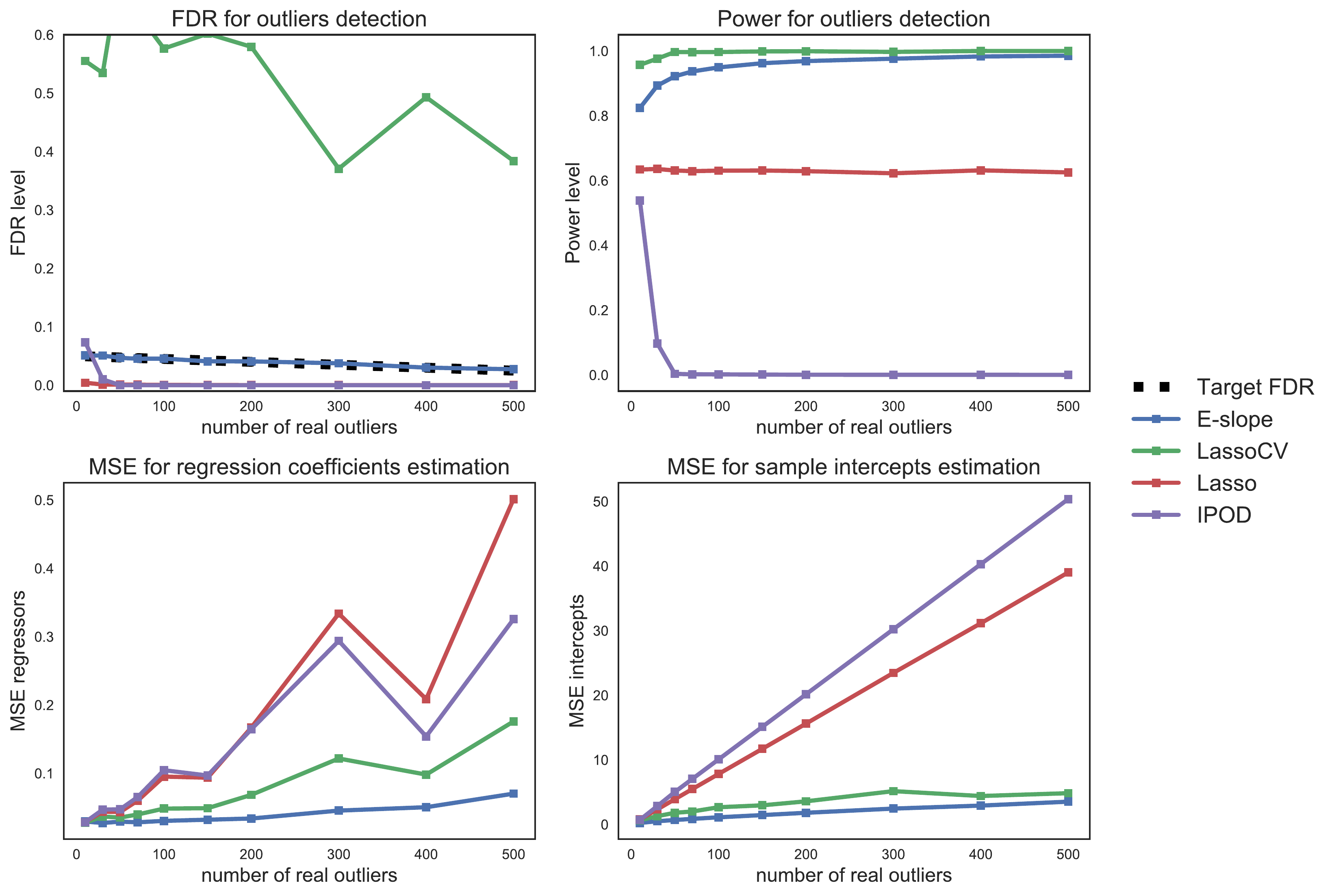}
\caption{Results for Simulation Setting~1 with low-magnitude outliers. First row gives the FDR (left) and power (right) of each considered procedure for outliers discoveries. 
Second row gives the MSE for regressors (left) and intercepts (right). 
Once again E-SLOPE nearly gives the best power, but is the only one to respect the required FDR, and provides the best MSEs.}
\label{fig:lowdim2}
\end{figure}
\begin{figure}
\centering
\includegraphics[width=\linewidth]{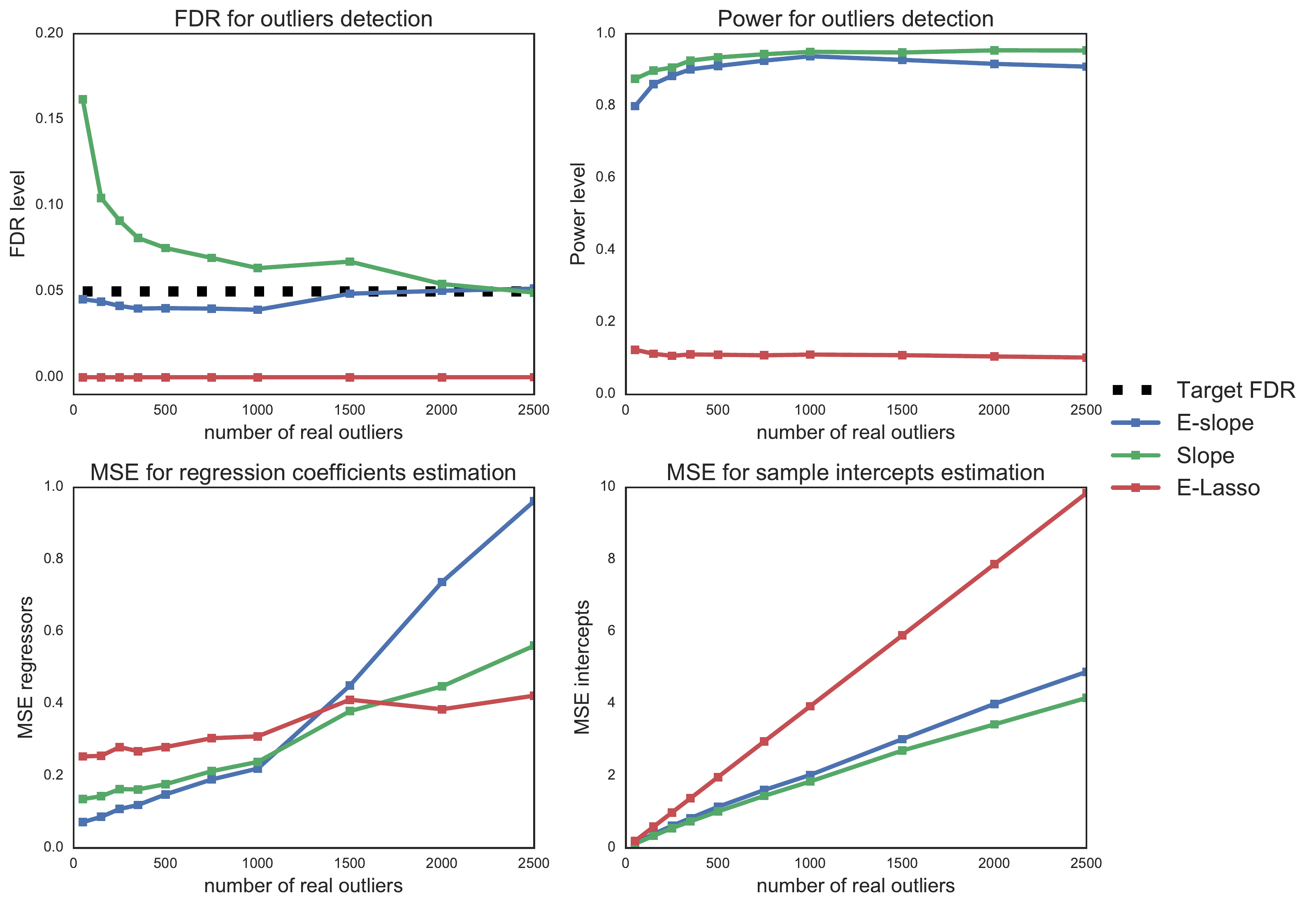}
\caption{Results for Simulation Setting~2 with low-magnitude outliers. First row gives the 
FDR (left) and power (right) of each considered procedure for outliers discoveries. 
Second row gives the MSE for regressors (left) and intercepts (right). 
Once again E-SLOPE nearly gives the best power, but is the only one to respect the required FDR.
It gives the best MSE for outliers estimation, and is competitive for regressors estimation.
All procedures have a poor MSE when the number of outliers is large, since the simulation setting 
considered in this experiment is hard: low-magnitude outliers and high-dimension.}
\label{fig:highdim1}
\end{figure}






\subsection{PGA/LPGA dataset}

This dataset contains Distance and Accuracy of shots, for PGA and LPGA players in 2008. This will allow us to visually compare the performance of IPOD, E-LASSO and E-SLOPE. Our data contain 197 points corresponding to PGA (men) players, to which we add~8 points corresponding to LPGA (women) players, injecting outliers. 
We apply SLOPE and LASSO on $\mu$ with several levels of penalization.
This leads to the ``regularization paths'' given in the top plots of Figure~\ref{fig:golf}, that shows the value of the 205 sample intercepts $\hat \mu$ as a function of the penalization level used in SLOPE and LASSO. Vertical lines indicate the choice of the parameter according the corresponding method (E-SLOPE, E-LASSO, IPOD).
We observe that E-SLOPE correctly discover the confirmed outliers (women data), together with two men observations that can be considered as outliers in view of the scatter plot. IPOD procedure does quite good, with no false discovery, but misses some real outliers (women data) and the suspicious point detected by E-SLOPE. E-LASSO does not make any false discovery but clearly reveals a lack of power, with only one discovery.

\begin{figure}[H]
\centering
\includegraphics[width=\linewidth]{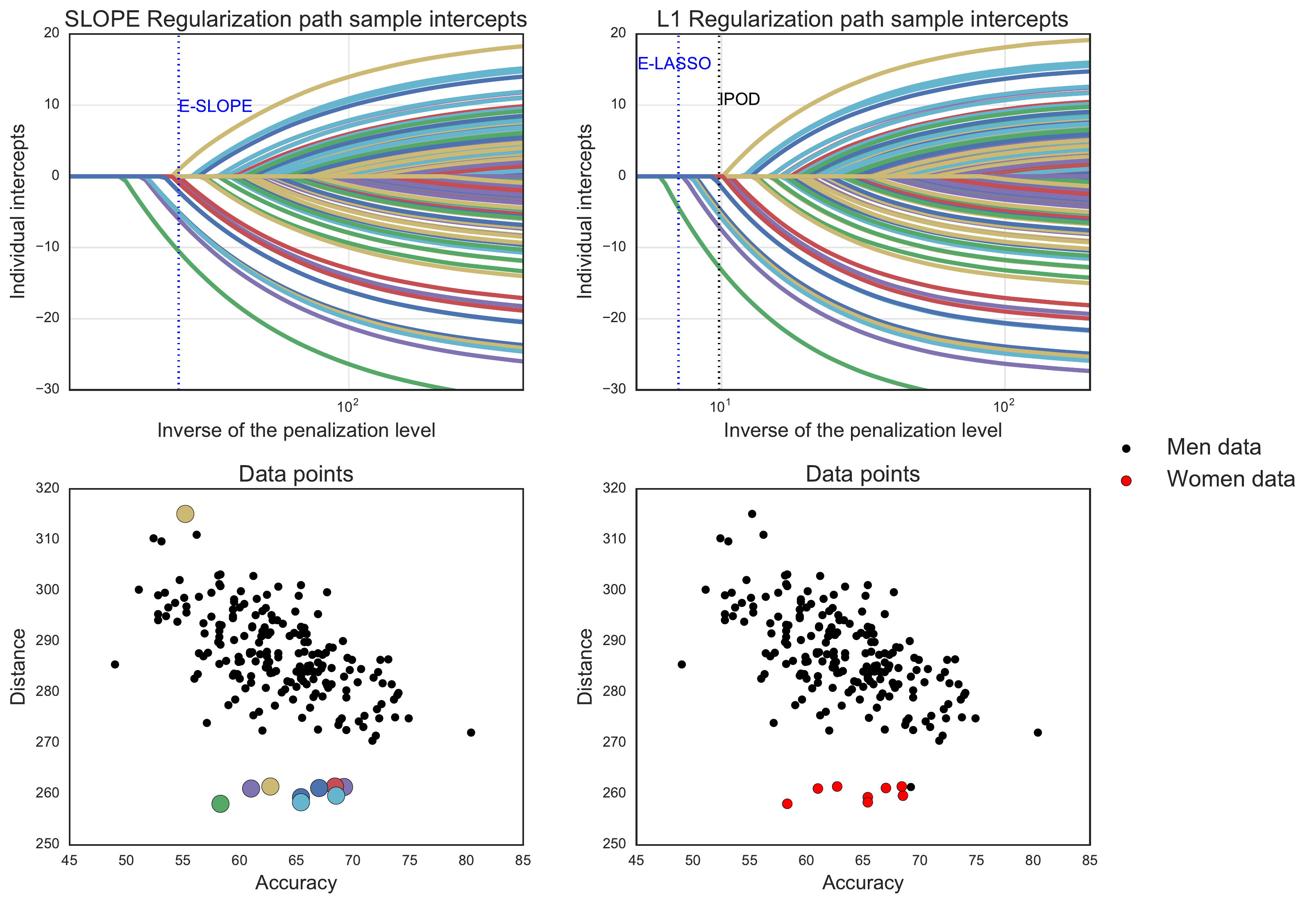}
\caption{PGA/LPGA dataset: top plots show the regularization paths for both types of penalization, bottom-left plot is a scatter plot of the data, with colored points corresponding to the discoveries made by E-SLOPE, bottom-right plot show the original data and the true outliers.}
\label{fig:golf}
\end{figure}

\subsection{Retail Sales Data}

This dataset is from the U.S. census Bureau, for year 1992.
The informations contained in it are the per capita retail sales of 845 US counties (in \$1000s). It also contains five covariates: the per capita retail establishments, the per capita income (in \$1000s), per capita federal expenditures (in \$1000s), and the number of males per 100 females.
No outliers are known, so we artificially create outliers by adding a small amount (magnitude $8$, random sign) to the retail sales of counties chosen uniformly at random. We consider various scenarii (from $1\%$ to $20\%$ of outliers) and compute the false discovery proportion and the power. Figure~\ref{fig:sales} below summarizes the results for the three procedures.

The results are in line with the fact that E-SLOPE is able to discover more outliers than its competitors. E-SLOPE has the highest power, and the FDP remains under the target level.

\begin{figure}[H]
\centering
\includegraphics[width=\linewidth]{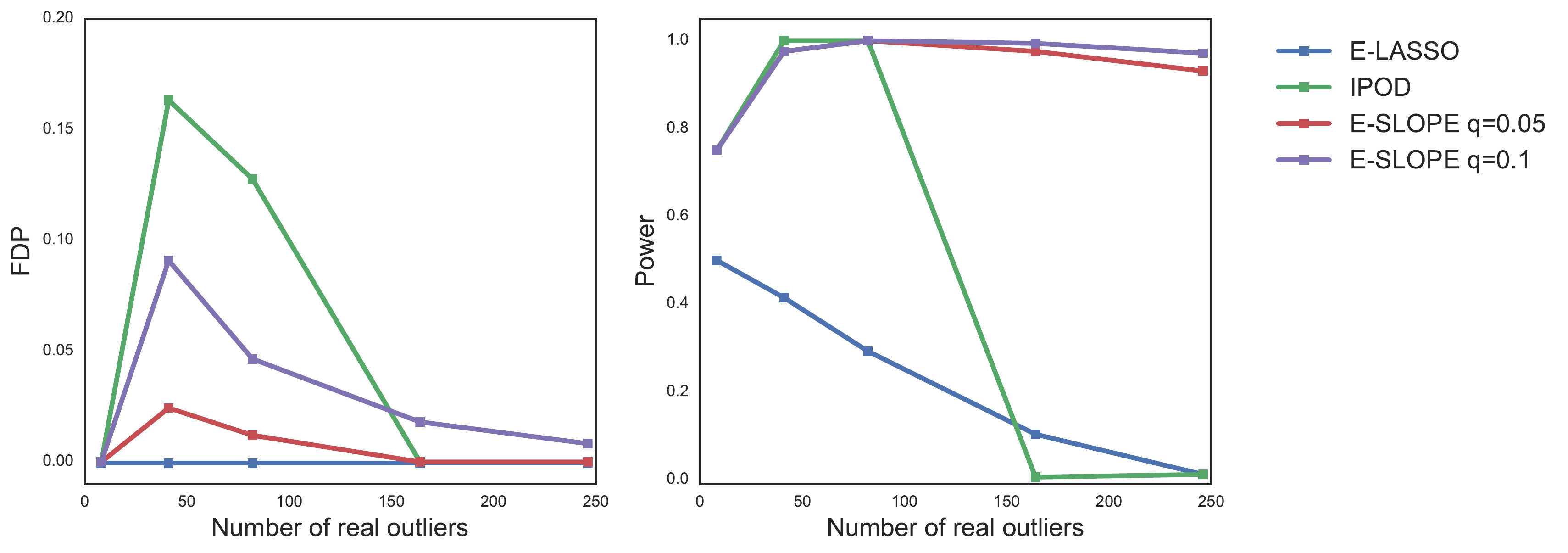}
\caption{\emph{Left}: False Discovery proportion, E-SLOPE remains under the target level; \emph{Right}:  power, E-SLOPE performs better than the competitors.}
\label{fig:sales}
\end{figure}

\subsection{Colorectal Cancer Data}

We consider whole exome sequencing data for 47 primary colorectal cancer tumors, characterized by a global genomic instability affecting repetitive DNA sequences (also known as microsatellite instable tumors, see \cite{bio1}).
In what follows, we restrict ourselves to repetitive sequences whose base motif is the single nucleotide A, and which are in regulatory regions (following the coding regions) that influence gene expression (UTR3). Same analysis could be run with different base motifs and different regions (exonic, intronic).
It has been shown in recent publications (see \cite{bio2}), that the probability of mutation of a sequence is dependent of the length of the repeat. 
So we fit, after a rescaled probit transformation, our mean-shift model with an intercept and the length of the repeat as covariates.  
The aim of the analysis is to find two categories of sequences: survivors (multi-satellites that mutated less than expected) and transformators (multi-satellites that mutated more than expected), with the idea that those sequences must play a key role in the cancer development.

We fix the FDR level $\alpha=5\%$, results are shown in Figure~\ref{fig:data_res}: blue dots are the observed mutation frequencies of each gene among the 47 tumors, plotted as a function of the repeat length of the corresponding gene, our discoveries are hightlighted in red.

We made 37 discoveries, and it is of particular interest to see that our procedure select both "obvious" outliers and more "challenging" observations that are discutable with the unaided eye. We also emphasize that with the IPOD procedure and the LASSO procedure described in the previous paragraph, respectively 32 and 22 discoveries were made, meaning that even with this stringent FDR level, our procedure allow us to make about $16\%$ more discoveries than IPOD.

\begin{figure}[H]
\centering
\includegraphics[width=0.8\textwidth]{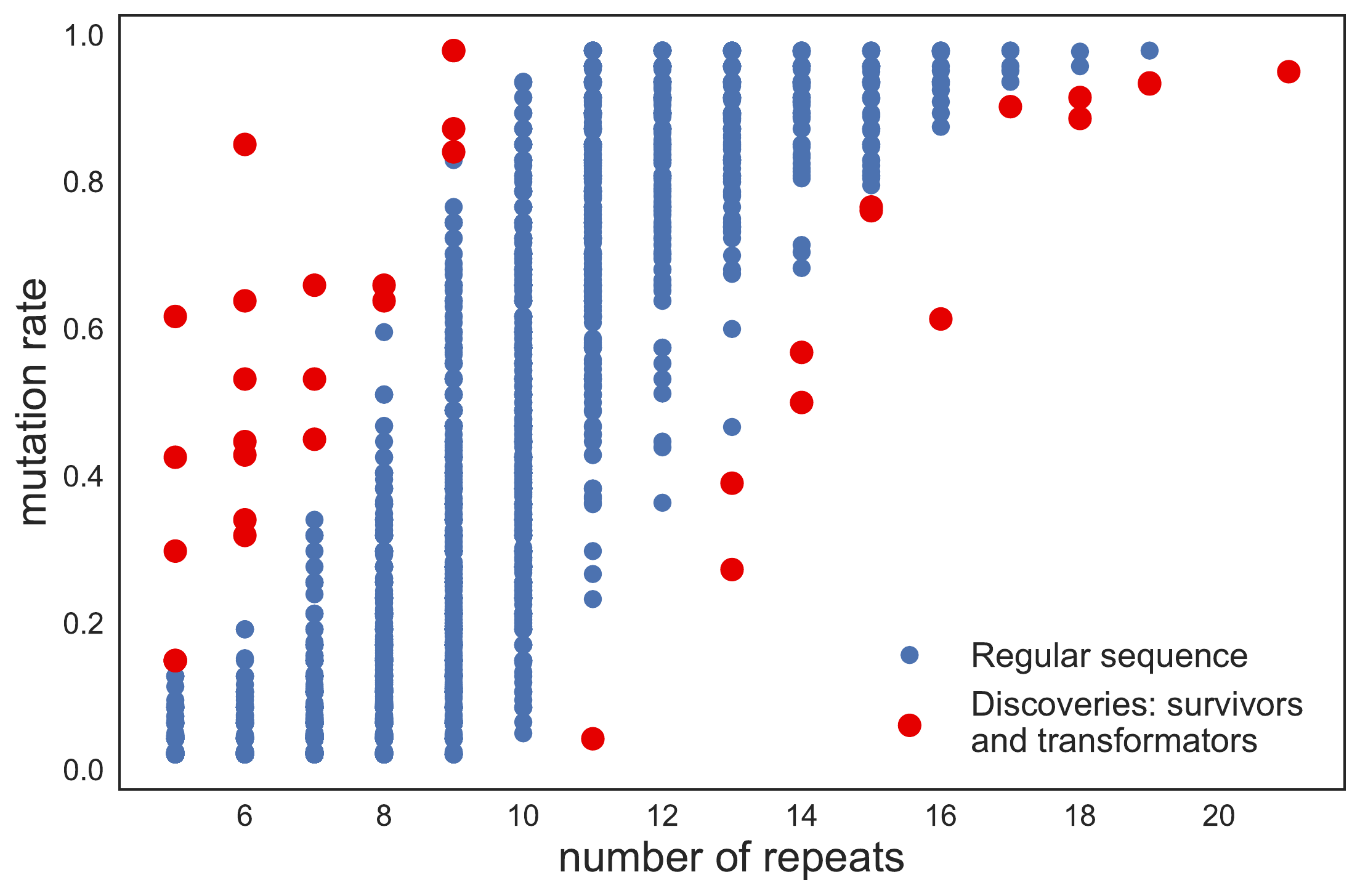}
\caption{Colorectal Cancer data: 37 discoveries made by E-SLOPE, namely sequence considered by our procedure as transformators or mutators (see main text for details). The procedure selects both ``obvious'' and more ``challenging'' observations for the eye. IPOD and Lasso procedures led only to 32 and 22 discoveries, despite the fact that we restricted E-SLOPE to the stringent FDR level of 5\%}
\label{fig:data_res}
\end{figure}

\section{Conclusion}

This paper introduces a novel approach for simultaneous robust estimation and outliers detection in the linear regression model. 
Three main results are provided: optimal bounds for the estimation problem in Section~\ref{SEC:BOUNDS}, that improve in particular previous results obtained with LASSO penalization~\cite{elasso}, and asymptotic FDR control and power consistency for the outlier detection problem in Section~\ref{sec:fdr}. 
To the best of our knowledge, this is the first result involving FDR control in this context. 

Our theoretical foundings are confirmed on intensive experiments both on real and synthetic datasets, showing that our procedure outperforms existing procedure in terms of power, while maintaining a control on the FDR, even in challenging situations such as low-magnitude outliers, a high-dimensional setting and highly correlated features.

Finally, this work extends the understanding of the deep connection between the SLOPE penalty and FDR control, previously studied in linear regression with orthogonal~\cite{slope} or i.i.d gaussian \cite{slopeminimax} features, which distinguishes SLOPE from other popular convex penalization methods.

\begin{appendix}

\section{Technical inequalities}
\label{appen:prel1}

The following technical inequalities are borrowed from \cite{Tsyb}, where proofs can be found.
Let $m,n,p$ be positive integers.
In the following lemma, an inequality for the sorted $\ell_1$-norm $J_\lambda$ (defined in equation \ref{sl1}) is stated.

\begin{lem}
For any two $x,y\in\R^m$, and any $s\in{1,\dots,m}$ such that $\abs{x}_0\leq s$ we have
\begin{equation*}
J_\lambda(x) - J_\lambda(y) \leq \Lambda (s) \norm{x-y}_2 - \sum_{j=s+1}^m \lambda_j \abs{x-y}_{(j)},
\end{equation*}
where
\begin{equation*}
\Lambda (s) = \sqrt{\sum_{j=1}^s \lambda_j^2}.
\end{equation*}
\label{lem:1}
\end{lem}
The following lemma gives a preliminary bound for the prediction error in our context, that are the starting point of our proof.
\begin{lem}
Let $h:\R^p \rightarrow \R$ be a convex function. Consider a $n\times p$ design matrix $X$, a vector  $\varepsilon\in\R^n$ and define $y=X\beta^\star+\varepsilon$ where $\beta^\star\in\R^p$. If $\hat{\beta}$ is a solution of the minimization problem $\min_{\beta\in\R^p} (\norm{y - X\beta}_2^2 + 2h(\beta))$, then $\hat{\beta}$ satisfies:
\begin{equation*}
\norm{X\hat{\beta} - X\beta^\star}_2^2 \leq \varepsilon^\top X(\hat{\beta} - \beta^\star) + h(\beta^\star) - h(\hat{\beta}).
\end{equation*}
\label{lem:2}
\end{lem}
\begin{proof}
Because the proof in \cite{Tsyb} is more general, we give a proof adapted to our context.
Optimality of $\hat{\beta}$ allows to choose $v$ in the subdifferential of $h$ sucht that
\begin{equation*}
0 = X^\top (X\hat{\beta} - y) + v = X^\top (X\hat{\beta} - X\beta^\star - \varepsilon) + v.
\end{equation*}
Therefore,
\begin{align*}
\norm{X\hat{\beta} - X\beta^\star}_2^2 &= (\hat{\beta} - \beta^\star)^\top X^\top X(\hat{\beta} - \beta^\star)  \\
&= (\hat{\beta} - \beta^\star)^\top (X^\top\varepsilon - v)\\
&= \varepsilon^\top X(\hat{\beta} - \beta^\star) + \langle v, \beta^\star - \hat{\beta} \rangle.
\end{align*}
Now, by definition of subdifferential, $h(\beta^\star) \geq h(\hat{\beta}) + \langle v, \beta^\star - \hat{\beta}  \rangle$. Combining this inequality with the previous equality leads to the conclusion.
\end{proof}

The following lemma allows to bound the inner product between a white Gaussian noise and any vector. The resulting bound involved the sorted $\ell_1$ norm.

\begin{lem}
Let $\delta_0\in (0,1)$ and let $X\in\R^{n\times p}$ with columns normed to~$1$. If $\varepsilon$ is $\mathcal{N}(0,I_n)$ distributed, then the event
\begin{equation*}
\left\{\forall u \in\R^p,\varepsilon^\top Xu \leq \max\left(H(u), G(u)\right) \right\}
\end{equation*}
is of probability at least $1-\delta_0/2$, where 
\begin{equation*}
H(u) = (4+ \sqrt{2})\sum_{j=1}^p \abs{u}_{(j)}\sigma \sqrt{\log(2p/j)}
\end{equation*}
 and 
\begin{equation*}
 G(u) = (4+\sqrt{2})\sigma \sqrt{\log(1/\delta_0)}\norm{u}_2.
 \end{equation*}
\label{lem:inner_product_bound}
\end{lem}

\section{Results related to Gaussian matrices}
\label{appen:prel2}

Inequalities for Gaussian random matrices are needed in what follows. They are stated here for the sake of clarity and we refer the reader to \cite{giraud} for proof (except bounds \ref{chi1} and \ref{chi2} that are taken from Lemma 1 in \cite{Massart}).
Again, $n$ and $p$ denote positive integers.

\begin{lem}
Let $X\in\R^{n\times p}$ with i.i.d $\mathcal{N}(0,I_p)$ rows. Denote by $\sigma_{max}$ the largest singular value of $X$. Then, for all $\tau \geq 0$,
\begin{equation}
\Proba \left( \frac{\sigma_{max}}{\sqrt{n}} \geq 1+ \sqrt{\frac{p}{n}} + \tau \right) \leq \exp\left(-\frac{n\tau^2}{2}\right).
\label{eq:sing}
\end{equation}
\label{lm:Gauss1}
\end{lem}

\begin{lem}
Concentration inequalities:
\begin{itemize}
\item Let $Z$ be $\mathcal{N}(0,1)$ distributed. Then for all $q\geq 0$:
\begin{equation}
\Proba\left( \abs{Z} \geq q \right) \leq \exp\left(-\frac{q^2}{2} \right).
\label{eq:gaus}
\end{equation}
\item Let $Z_1,Z_2,\dots,Z_p$ be independent and $ \mathcal{N}(0,\sigma^2)$ distributed. Then for all $L > 0$:
\begin{equation}
\Proba\left( \max_{i=1,\dots,p}\abs{Z_i} > \sigma\sqrt{2\log p +2L}  \right) \leq e^{-L}.
\label{eq:maxgaus}
\end{equation}
\item Let $X$ be $\chi_2(n)$ distributed. Then, for all $x>0$:
\begin{equation}
\Proba\left( X-n \geq 2\sqrt{nx} +2x\right) \leq \exp(-x).
\label{chi1}
\end{equation}
\begin{equation}
\Proba\left( n-X \geq 2\sqrt{nx} \right) \leq \exp(-x).
\label{chi2}
\end{equation}
\end{itemize}
\label{lm:Gauss2}
\end{lem}

The following recent result (\cite{REgauss}, Theorem~1) will also be useful.
\begin{lem}
Let $X\in\R^{n\times p}$ with i.i.d $\mathcal{N}(0,\Sigma)$ rows. There exists positive constants $c$ and $c'$ such that with probability greater than $1-c'\exp(-cn)$, we have for all $z\in\R^p$:
\begin{equation}
\frac{\norm{Xz}_2}{\sqrt{n}} \geq \frac{1}{4}\sqrt{\lambda_{min}(\Sigma)}\norm{z}_2 - 9\sqrt{\max_j \Sigma_{jj} \frac{\log p}{n}}\norm{z}_1,
\label{eq:re_gauss}
\end{equation}
where $\lambda_{min}(\Sigma)$ is the lowest eigenvalue of $\Sigma$.
\label{lm:re_gauss}
\end{lem}

\section{Proof of Section \ref{SEC:BOUNDS}}
\label{Appen:p1}

This section is devoted to the proof of our main results, stated in Section~\ref{SEC:BOUNDS}.

\subsection{Proof of Theorem \ref{thm:RE}}
\label{Appen:proof_RE}

Define $D$ the diagonal matrix such that $X=X'D$ ($D$ is the diagonal matrix formed by the inverse of the norm of each column of $X'$). Applying now Lemma~\ref{lm:re_gauss} for $X'$ and $Dz$ we obtain for all $z\in\R^p$
\begin{align*}
\norm{Xz}_2 &\geq  \frac{1}{4}\sqrt{\lambda_{min}(\Sigma)} \norm{\sqrt{n}Dz}_2 - 9\sqrt{\max_j \Sigma_{jj}\frac{\log p}{n}}\norm{\sqrt{n}Dz}_1 \\
&\geq \frac{\sqrt{n}\sqrt{\lambda_{min}(\Sigma)}}{4M}\norm{z}_2 - 9\frac{\sqrt{\max_j \Sigma_{jj}\log p}}{m}\norm{z}_1, 
\end{align*}
with probability greater than $1-c'\exp(-cn)$,
where $M$ and $m$ denote respectively the maximum and minimum of the norms of the columns of $X'$. Note that for all $1\leq i \leq p$, the squared norm of the $i^{th}$ column of $X'$ is $\sigma_i^2\chi_2 (n)$ distributed, so using the bounds \ref{chi1} and \ref{chi2} of Lemma \ref{lm:Gauss2} (respectively with $x=n$ and $x=n/16$), together with a union bound we obtain that with probability greater than $1-ne^{-n} - ne^{-n/16}$
\begin{equation*}
M\leq (\max_j \Sigma_{jj})\sqrt{5n}, \qquad m\geq (\min_j \Sigma_{jj})\sqrt{\frac{n}{2}},
\end{equation*}
and we eventually obtain
\begin{equation}
\norm{Xz}_2 \geq \frac{\sqrt{\lambda_{min}(\Sigma)}}{4\sqrt{5}\max_j \Sigma_{jj}}\norm{z}_2 - \frac{9}{\min_j \Sigma_{jj}}\sqrt{\frac{2\log p\max_j \Sigma_{jj}}{n}}\norm{z}_1.
\label{eq:Xn}
\end{equation}
Let define $v = [\beta^\top, \mu^\top]^\top \in \R^{p + n} \in\mathcal{C}(k,s, c_0)$ (see Definition \ref{Cewre}). Then,
\begin{equation}
\norm{\beta}_1 \leq \sum_{j=1}^p \frac{\tilde{\lambda}_j}{\tilde{\lambda}_p}\abs{\beta}_{(j)}\leq (1+c_0)\left(\sqrt{k}\norm{\beta}_2 + \sqrt{s}\norm{\mu}_2\right).
\end{equation}
Thus we obtain 
\begin{equation}
\norm{\beta}_1 \leq  (1+c_0)\left(\sqrt{k}\norm{\beta}_2 + \sqrt{s}\norm{\mu}_2
\right).
\label{eq:D2}
\end{equation}
Injecting \eqref{eq:D2} in \eqref{eq:Xn} applied to the vector $\beta$ now leads to 
\begin{multline}
\norm{X\beta}_2 + \norm{\mu}_2 \geq  \norm{\beta}_2 \Big( \frac{\sqrt{\lambda_{min}(\Sigma)}}{4\sqrt{5}\max_j \Sigma_{jj}} - \frac{9}{\min_j \Sigma_{jj}}(1+c_0) \sqrt{\frac{2k(\log p)\max_j \Sigma_{jj}}{n}} \Big) \\ 
+ \norm{\mu}_2\Big(1- \frac{9}{\min_j \Sigma_{jj}}(1+c_0) \sqrt{\frac{2s(\log p)\max_j \Sigma_{jj}}{n}} \Big).
\label{eq:D3}
\end{multline}
For $n$ large enough as explicited in the assumption of Theorem~\ref{thm:RE}, Equation~\eqref{eq:D3} turns to
\begin{equation*}
\norm{X\beta}_2 + \norm{\mu}_2 \geq \frac{\sqrt{\lambda_{min}(\Sigma)}}{8\sqrt{5}\max_j \Sigma_{jj}}\norm{\beta}_2 +  \frac{1}{2} \norm{\mu}_2,
\end{equation*}
and thus, using the fact that $2(a^2+b^2)\geq (a+b)^2$,
\begin{equation}
\norm{X\beta}_2^2 + \norm{\mu}_2^2 \geq \min\left\{\frac{\lambda_{min}(\Sigma)}{128\times 5(\max_j \Sigma_{jj})^2}, \frac{1}{8}\right\} \norm{v}_2^2.
\end{equation}
Now if $v = [\beta^\top, \mu^\top]^\top \in \R^{p + n} \in\mathcal{C}^p(s, c_0)$, Equation~\eqref{eq:Xn} together with the inequality $\norm{\beta}_1 \leq \sqrt{p}\norm{\beta}_2 $ lead to 
\begin{equation*}
\norm{X\beta}_2 + \norm{\mu}_2 \geq  \norm{\beta}_2 \Big( \frac{\sqrt{\lambda_{min}(\Sigma)}}{4\sqrt{5}\max_j \Sigma_{jj}} - \frac{9}{\min_j \Sigma_{jj}} \sqrt{\frac{2p\log p \max_j \Sigma_{jj}}{n}} \Big)
+ \norm{\mu}_2,
\end{equation*}
and we conclude as above.
Thus the first part of the theorem is satisfied.

Now, we must lower bound the scalar product $\langle X\beta, \mu \rangle $.\\ Divide $\{1,\dots,p\} = T_1\cup T_2\cup\dots \cup T_t$ with $T_i$ ($1\leq i \leq t-1$) of cardinality~$k'$ containing the support of the $k'$ largest absolute values of $b_{\left(\bigcup_{j=1}^{i-1} T_j\right)^c}$ and~$T_t$ of cardinality $k''\leq k'$ the support of the remaining values. Divide in the same way $\{1,\dots,n\} = S_1\cup S_2\cup\dots \cup S_q$ (of cardinalitys $s', \dots, s', s''\leq s'$) with respect to the largest absolute values of $\mu$ ($k'$ and $s'$ to be chosen later).
We use this to lower bound the scalar product:
\begin{equation*}
\vert\langle X\beta, \mu \rangle\vert =  \vert\langle X'D\beta, \mu \rangle\vert \leq \sum_{i=1}^q\sum_{j=1}^t \vert \langle X'_{S_i,T_j} (D\beta)_{T_j}, \mu_{S_i} \rangle \vert ,
\end{equation*}
so
\begin{equation}
\abs{\langle X\beta, \mu \rangle} \leq \max_{i,j} \nor{X'_{S_i,T_j}}_2 \frac{1}{m}\sum_{j=1}^t \nor{\beta_{T_j}}_2 \sum_{i=1}^q\nor{\mu_{S_i}}_2 ,
\label{eq:sp}
\end{equation}
where we recall that $m$ is the minimal value of the column norms of $X'$.
According to Lemma~\ref{lm:Gauss1}, conditionnally on $S_i$ and $T_j$, we have with probability greater than $1-\exp(-n\tau^2 /2)$,
\begin{equation*}
\nor{X'_{S_i,T_j}}_2 \leq \Vert\Sigma^{1/2}_{T_j,T_j}\Vert\big(\sqrt{k'} + \sqrt{s'} + \sqrt{s'}\tau  \big) \leq \sqrt{\lambda_{max}(\Sigma)}\big(\sqrt{k'} + \sqrt{s'} + \sqrt{s'}\tau  \big).
\end{equation*}
Considering all possibilities for $S_i$ and $T_j$, we have with probability greater than $1-\dbinom{p}{k'}\dbinom{n}{s'}e^{-n\tau ^2/2}$,
\begin{equation}
\max_{i,j}\nor{X'_{S_i, T_j}}_2 \leq \sqrt{\lambda_{max}(\Sigma)}\big(\sqrt{k'} + (1+\tau)\sqrt{s'} \big).
\label{eq:D7}
\end{equation}
Moreover, thanks to the decreasing value along the subset $T_j$ we can use the trick of \cite{CandesPlan}, writing for all $j\in\{1,\dots,t-1\}$ and all $x\in\{1,\dots,\abs{T_{j+1}}\}$: 
\begin{equation*}
\abs{\big(\beta_{T_{j+1}}\big)_x}\leq  \frac{\nor{\beta_{T_j}}_1}{\abs{T_j}}.
\end{equation*}
Squaring this inequality and summing over $x$ gives:
\begin{equation*}
\norm{\beta_{T_{j+1}}}_2^2 \leq  \frac{\norm{b_{T_j}}_1^2}{\abs{T_j}}\frac{\abs{T_{j+1}}}{\abs{T_j}} \leq \frac{\norm{b_{T_j}}_1^2}{\abs{T_j}} = \frac{\norm{\beta_{T_j}}_1^2}{k'}.
\end{equation*}
Then, 
\begin{equation*}
\sum_{j=1}^t \norm{\beta_{T_j}}_2 \leq \norm{\beta}_2 + \sum_{j=2}^t \norm{\beta_{T_j}}_2 \leq \norm{\beta}_2 + \frac{1}{\sqrt{k'}}\sum_{j=1}^{t-1} \norm{\beta_{T_j}}_1 \leq  \norm{\beta}_2  + \frac{1}{\sqrt{k'}}\norm{\beta}_1,
\end{equation*}
and so
\begin{equation}
\sum_{j=1}^t \norm{\beta_{T_j}}_2 \leq  \norm{\beta}_2  + \frac{1}{\sqrt{k'}}\sum_{j=1}^p \frac{\tilde{\lambda}_j}{\tilde{\lambda}_p}\abs{\beta}_{(j)}.
\label{eq:D8}
\end{equation}
In the same way we obtain:
\begin{equation}
\sum_{i=1}^q\norm{\mu_{S_i}}_2 \leq \norm{\mu}_2 + \frac{1}{\sqrt{s'}}\sum_{j=1}^n \frac{\lambda_j}{\lambda_n}\abs{\mu}_{(j)}.
\label{eq:D9}
\end{equation}
Now if $v = [\beta^\top, \mu^\top]^\top \in \R^{p + n} \in\mathcal{C}(k,s, c_0)$,
\begin{align*}
\sum_{j=1}^t \norm{\beta_{T_j}}_2 \sum_{i=1}^q\norm{\mu_{S_i}}_2 &\leq \big( \norm{\beta}_2 +\frac{1}{\sqrt{k'}} (1+c_0)(\sqrt{k}\norm{\beta}_2 + \sqrt{s}\norm{\mu}_2) \big)\\
&\quad \times \big( \norm{\mu}_2 +\frac{1}{\sqrt{s'}} (1+c_0)(\sqrt{k}\norm{\beta}_2 + \sqrt{s}\norm{\mu}_2) \big)\\
&\leq \left(2\norm{\beta}_2 + \norm{\mu}_2\right)\left(2\norm{\mu}_2 + \norm{\beta}_2\right)\\
&\leq 2\norm{v}_2^2 + 5\norm{\mu}_2\norm{\beta}_2\\
&\leq 5\norm{v}_2^2, 
\end{align*}
where we chose $k' = s' = (1+c_0)^2(k \vee s)$.
Combining this last inequality with Equations \eqref{eq:sp} and \eqref{eq:D7}, and using again that $m\geq \min_j \Sigma_{jj}\sqrt{n/2}$ with probability greater than $1-ne^{-n/16}$, lead to
\begin{equation}
\abs{\langle X\beta, \mu \rangle}  \leq \frac{\sqrt{\lambda_{max}(\Sigma)}}{\min_j \Sigma_{jj}}(2+\tau)\sqrt{\frac{2s'}{n}} 5\norm{v}_2^2.
\label{eq:D10}
\end{equation}
Note that with this choice of $s'$ and $k'$, the assumptions on $n$ and the constant~$C'$ defined in the theorem lead to $\dbinom{p}{k'}\leq (ep/k')^{k'} \leq \exp(n/C')$, and~$\dbinom{n}{s'}\leq (en/s')^{s'} \leq \exp(n/C')$ , so we have Equation \eqref{eq:D7} with probability greater than $1-\exp\left(-n\left(\tau^2 /2 - 2C'^{-1}\right) \right)$.
With the specific assumption on $n$ in the statement of the theorem, the term in the right part of Equation \eqref{eq:D10} is small enough to obtain:
\begin{equation}
2\abs{\langle Xb, u \rangle}  \leq \min\left\{\frac{\lambda_{min}(\Sigma)}{256\times 5 \max_j \Sigma_{jj}}, \; \frac{1}{16} \right\} \norm{v}_2^2
\label{eq:D11}
\end{equation}
Eventually, if $v = [\beta^\top, \mu^\top]^\top \in \R^{p + n} \in\mathcal{C}^p(s, c_0)$, Equation~\eqref{eq:D9} still holds, and combining it with Equation~\eqref{eq:D7} and Equation~\eqref{eq:sp} with $t=1$ leads to:
\begin{multline}
\abs{\langle X\beta, \mu \rangle}  \leq \frac{\sqrt{\lambda_{max}(\Sigma)}}{\min_j \Sigma_{jj}}(\sqrt{p} + (1+\tau)\sqrt{s'})\sqrt{\frac{2}{n}} \norm{\beta}_2 \\
\times \big( \norm{\mu}_2 +\frac{1}{\sqrt{s'}} (1+c_0)\left(\sqrt{p}\norm{\beta}_2 + \sqrt{s}\norm{\mu}_2\right) \big).
\label{eq:D12}
\end{multline}
Choosing $s' = (1+c_0)^2(p \vee s)$,
\begin{align*}
\abs{\langle X\beta, \mu \rangle}  &\leq \frac{\sqrt{\lambda_{max}(\Sigma)}}{\min_j \Sigma_{jj}}(2+\tau)\sqrt{\frac{2s'}{n}} \norm{\beta}_2 \left( 2\norm{\mu}_2 + \norm{\beta}_2 \right)\\
& \leq \frac{\sqrt{\lambda_{max}(\Sigma)}}{\min_j \Sigma_{jj}}(2+\tau)\sqrt{\frac{2s'}{n}}2\norm{v}_2^2.
\end{align*}
 We conclude as above, thus leading to the second part of the theorem.
 
\subsection{Proof of Theorem \ref{thm:upper_bound_nobeta}}
\label{proof:upper_bound_nobeta}
 
 We will actually show a slightly more general result. Let $R$ be any subset of cardinality $r$ containing the support of the true parameter $\mu^\star$ and $I_R$ be the matrix obtained by extracting columns with indices in $R$ from the identity matrix. We consider the following minimization:
 $$
 \hat{\beta}, \hat{\mu} = \argmin_{\beta, \mu} \Vert y - X\beta - I_R \mu \Vert_2 ^2 + 2\rho J_{\lambda^{[r]}} (\mu),
 $$
 where $\lambda^{[r]}$ contains the first $r$ terms of the sequence of weights defined in Section \ref{SEC:BOUNDS}. 
 Obviously, the theorem will result from the case $P=\{1,\dots,n\}$. Note that $\hat{\mu}$ belongs to $\R^r$.
 
Defining $b=\hat{\beta}-\beta^\star$ and $u=I_R (\hat{\mu}-\mu_R^\star)$ where $\mu_R^\star$ denotes the vector extracted from $\mu^\star$ selecting coordinates corresponding to indices in $R$ (note that the eliminated coordinates are zeros),  we can apply Lemma \ref{lem:2} to obtain:
\begin{align*}
\norm{Xb+ u}_2^2 &\leq \varepsilon^\top (Xb +  u) + \rho J_{\lambda^{[r]}}(\mu^\star) - \rho J_{\lambda^{[r]}}(\hat{\mu}) \\&= \varepsilon^\top (Xb +  u) + \rho J_\lambda(I_R\mu_R^\star) - \rho J_\lambda(I_R\hat{\mu}).
\end{align*}
Note that it is crucial to have $\supp(\mu^\star)\subset R$ in order to write $\mu^\star = I_R \mu_R^\star$.
Applying now Lemma \ref{lem:1} we obtain:
\begin{equation}
\norm{Xb +  u}_n^2 \leq \varepsilon^\top (Xb + u) + \rho\big(\Lambda(s)\norm{u}_2 - \sum_{j=s+1}^n \lambda_j \abs{u}_{(j)} \big),
\label{eq:C1}
\end{equation}
where $\Lambda(s)$ is defined as $\sqrt{\sum_{j=1}^s \lambda_j^2}$.
Hence, using Cauchy-Scwarz inequality we get:
\begin{equation*}
\norm{Xb+  u}_2^2 \leq \nor{X^\top \varepsilon}_2 \norm{b}_2 + \varepsilon^\top u + \rho\big(\Lambda(s)\norm{u}_2 - \sum_{j=s+1}^n \lambda_j \abs{u}_{(j)} \big).
\end{equation*}
Then, by Lemma~\ref{lem:inner_product_bound}, with probability greater than $1-\delta_0/2$ we have (the last inequality is used for the sake of simplicity):
\begin{equation*}
	\varepsilon^\top u \leq \max (H(u), G(u)) \leq H(u) + G(u), 	
\end{equation*}
with $H(u)$ and $G(u)$ defined in Lemma~\ref{lem:inner_product_bound}. Additionnally, $\frac{1}{\sigma^2}\norm{X^\top \varepsilon}_2^2$ follows a $\chi^2$ law with $p$ degrees of freedom, so by the third point in Lemma~\ref{lm:Gauss2} with $x=L p$ this provides, chosing $\delta_0 = (s/2n)^s$, that with probability greater than $1-\frac{1}{2}(s/2n)^s- \exp(-L p)$:
\begin{align*}
\norm{Xb+ u}_2^2 &\leq c_L\sigma\sqrt{p} \norm{b}_2 + H(u) + G(u) + \rho\big(\Lambda(s)\norm{u}_2 - \sum_{j=s+1}^n \lambda_j \abs{u}_{(j)} \big) \\
&\leq c_L \sigma \sqrt{p}\norm{b}_2 + \frac{\rho}{2}\sum_{j=1}^n \lambda_j \abs{u}_{(j)} \\
&\quad + \frac{\rho}{2}\sqrt{s\log(2n/s)}\norm{u}_2 + \rho\big(\Lambda(s)\norm{u}_2 - \sum_{j=s+1}^n \lambda_j \abs{u}_{(j)} \big) \\
&\leq c_L \sigma \sqrt{p}\norm{b}_2 + \big(2\rho\Lambda(s)\norm{u}_2 - \frac{\rho}{2} \sum_{j=s+1}^n \lambda_j \abs{u}_{(j)} \big),
\end{align*}
where $c_L = \sqrt{1+2L + 2\sqrt{L}}$ and where we used Equation~\eqref{eq:sum_log_inequality} to obtain the last inequality.
The fact that the left part of the last inequality is positive gives:
$$
\sum_{j=1}^n \lambda_j \abs{u}_{(j)} \leq  \sum_{j=s+1}^n \lambda_j \abs{u}_{(j)} + \Lambda(s)\norm{u}_2  \leq \frac{2}{\rho}c_L \sigma\sqrt{p}\norm{b}_2 + 5\Lambda(s)\norm{u}_2, 
$$
where the left part of the inequality is obtained using Cauchy-Schwarz inequality.
Hence,
\begin{equation}
\sum_{j=1}^n \frac{\lambda_j}{\lambda_n} \abs{u}_{(j)} \leq \frac{2c_L}{\rho} \sqrt{\frac{p}{\log 2}}\norm{b}_2 + 5\sqrt{\frac{s\log\left(2en/s \right)}{\log 2}} \norm{u}_2,
\end{equation}
where we used the right part of the inequality~\eqref{eq:sum_log_inequality}.
Choosing $L=1$ lead to $c_L = \sqrt{5}$, and reminding that $\rho \geq 2(4+\sqrt{2})$ we conclude that $ [b^\top, u^\top]^\top \in \mathcal{C}^p(s_1,4)$ (see Definition~\ref{Cewre}) with $s_1 = \frac{s\log\left(2en/s\right)}{\log 2}$.
Therefore, by Condition~\ref{Hyp} and the definition of $\kappa$ therein :
\begin{align*}
2\norm{Xb+ u}_2^2 &\leq 2\sqrt{5}\sigma\sqrt{p}\norm{b}_2 + 4\rho 
\Lambda(s)\norm{u}_2 \\
&\leq  \frac{5\sigma^2}{\kappa^2}p + \kappa^2\norm{b}_2^2 + 
\frac{4\rho^2 \Lambda(s)^2}{\kappa^2} + \kappa^2\norm{u}_2^2\\
&\leq \frac{4\rho^2 }{\kappa^2}\Lambda(s)^2 + \frac{5\sigma^2}{\kappa^2}p 
+ \kappa^2 \norm{v}_2^2\\
& \leq \frac{4\rho^2 }{\kappa^2}\Lambda(s)^2 + \frac{5\sigma^2}{\kappa^2}p 
+ \norm{Xb+ u}_2^2 .
\end{align*}
Thus,
\begin{equation*}
\norm{Xb+u}_2^2 \leq \frac{4\rho^2 }{\kappa^2}\Lambda(s)^2 
+ \frac{5\sigma^2}{\kappa^2}p,
\end{equation*}
and 
\begin{equation*}
\norm{b}_2^2 + \norm{u}_2^2 \leq \frac{4\rho^2}{\kappa^4}\Lambda(s)^2 + \frac{5\sigma^2 }{\kappa^4}p.
\end{equation*}
The proof of Theorem \ref{thm:upper_bound_nobeta} concludes by the classical inequalities~\cite{Tsyb}
\begin{equation}
	\label{eq:sum_log_inequality}
	s \log\Big(\frac{2n}{s} \Big) \leq \sum_{j=1}^s \log \Big(\frac{2n}{j} \Big) 
	= s\log(2n) - \log(s!) \leq s\log\Big(\frac{2en}{s}\Big).
\end{equation}

\subsection{Proof of Theorem \ref{thm:upper_bound_beta_l1}}

As in the previous proof, the more general version still holds and in the same way we obtained \eqref{eq:C1}, with the same definition of $b$ and $u$, we now have:
\begin{equation*}
\norm{Xb+ u}_2^2 \leq \varepsilon^\top (Xb + u) + \nu (\norm{\beta^\star}_1 - \Vert\hat{\beta}\Vert_1) + \rho (\Lambda(s)\norm{u}_2 - \sum_{j=s+1}^n \lambda_j \abs{u}_{(j)}).	
\end{equation*}
With $T$ being the support of the true regression vector $\beta^\star$ we have, using the triangle inequality:
$$
\norm{\beta^\star}_1 - \Vert\hat{\beta}\Vert_1 = \norm{\beta^\star_T}_1 - \Vert b+\beta^\star\Vert_1 = \norm{\beta^\star_T}_1 - \norm{b_T + \beta^\star_T}_1 - \norm{b_{T^c}}_1 \leq \norm{b_T}_1 - \norm{b_{T^c}}_1.
$$
Hence we can write:
\begin{align*}
\norm{Xb+ u}_2^2 & \leq \nor{X^\top \varepsilon}_\infty \norm{b}_1 + \nu \left(\norm{b_T}_1 - \norm{b_{T^c}}_1 \right) + \varepsilon^\top u \\
&\quad + \rho\Lambda(s)\norm{u}_2 - \rho\sum_{j=s+1}^n \lambda_j \abs{u}_{(j)}\\
& \leq  \norm{b_T}_1(\nu + \nor{X^\top \varepsilon}_\infty) - \norm{b_{T^c}}_1( \nu - \nor{X^\top \varepsilon}_\infty ) + \varepsilon^\top u \\
&\quad + \rho\Lambda(s)\norm{u}_2  - \rho\sum_{j=s+1}^n \lambda_j \abs{u}_{(j)}.
\end{align*}
With the choice  $\nu=4\sigma\sqrt{\log p}$ we have $\nor{X^\top \varepsilon}_\infty \leq \nu/2$ according to Lemma~\ref{lm:Gauss2}, with probability greater than $1-\frac{1}{p}$. Using again Lemma \ref{lem:inner_product_bound} to bound $\varepsilon^\top u $, we obtain that with probability greater than $1-\frac{1}{2}\left(\frac{s}{2n}\right)^s - \frac{1}{p}$:
\begin{equation}
\norm{Xb+ u}_2^2 \leq \norm{b_T}_1 (6\sigma\sqrt{\log p}) - \norm{b_{T^c}}_1 (2\sigma\sqrt{\log p} ) +  2\rho\Lambda(s)\norm{u}_2 - \frac{\rho}{2}\sum_{j=s+1}^n \lambda_j \abs{u}_{(j)}.
\label{eq:C3}
\end{equation}
The fact that the left part of the inequality is positive gives:
$$
\frac{4}{\rho}\sigma\sqrt{\log p}\norm{b_{T^c}}_1 + \sum_{j=s+1}^n \lambda_j \abs{u}_{(j)} \leq  \frac{12}{\rho}\sigma\sqrt{\log p}\norm{b_T}_1 + 4\Lambda(s)\norm{u}_2,
$$
and using Cauchy-Schwarz inequality, this leads to:
$$
\frac{4}{\rho}\sigma\sqrt{\log p}\norm{b}_1 + \sum_{j=1}^n \lambda_j \abs{u}_{(j)} \leq \frac{16}{\rho}\sigma\sqrt{k\log p}\norm{b}_2 + 5\Lambda(s)\norm{u}_2
$$
Eventually we obtain, because $\lambda_n = \sigma\sqrt{\log 2}$ and $\sqrt{\log p} \geq \frac{\rho \log 2}{4}$:
\begin{equation}
\label{eq:cone_lasso_slope}
\norm{b}_1 +  \sum_{j=1}^n \frac{\lambda_j}{\lambda_n} \abs{u}_{(j)} \leq \frac{4\sigma\sqrt{\log p}}{\rho\lambda_n}\norm{b}_1 + \sum_{j=1}^n \frac{\lambda_j}{\lambda_n} \abs{u}_{(j)}\leq \frac{16\sigma\sqrt{k\log p}}{\rho\lambda_n}\norm{b}_2 + \frac{5\Lambda(s)}{\lambda_n}\norm{u}_2
\end{equation}
and the concatenated vector of $b$ and $u$ is therefore in the cone $\mathcal{C}(k_1,s_1,4)$ with $k_1 = 16k\log p/\log 2$ and $s_1 = s\log (2en/s)/\log 2$.
Starting from \eqref{eq:C3}, we obtain, using again $\kappa$ as the capacity constant in Assumption \ref{Hyp}:
\begin{align*}
2\norm{Xb+u}_2^2 &\leq \norm{b_T}_1 12\sigma\sqrt{\log p} + 4\rho\Lambda(s)\norm{u}_2\\
& \leq 12\sigma\sqrt{k\log p}\norm{b}_2 + 4\rho\Lambda(s)\norm{u}_2\\
& \leq \frac{36}{\kappa^2}\sigma^2 k\log p + \kappa^2\norm{b}_2^2 + \frac{4\rho^2}{\kappa^2}\Lambda(s)^2 + \kappa^2 \norm{u}_2^2 \\
&\leq \frac{36}{\kappa^2}\sigma^2 k\log p + \frac{4\rho^2}{\kappa^2}\Lambda(s)^2 + \kappa^2 \norm{v}_2^2\\
& \leq \frac{36}{\kappa^2}\sigma^2 k\log p + \frac{4\rho^2}{\kappa^2}\Lambda(s)^2 + \norm{Xb+ u}_2^2. 
\end{align*}
Thus,
$$
\norm{Xb+ u}_2^2 \leq \frac{36}{\kappa^2}\sigma^2 k\log p + \frac{4\rho^2}{\kappa^2}\Lambda(s)^2
$$
and using again Assumption \ref{Hyp} and the remark after:
$$
\norm{b}_2^2 + \norm{u}_2^2 \leq \frac{36}{\kappa^4}\sigma^2 k\log p + \frac{4}{\kappa^4}\Lambda(s)^2
$$

\subsection{Proof of Theorem \ref{thm:upper_bound_beta_sl1}}
\label{proof:beta_sl1}

In the same way we obtained \eqref{eq:C1}, we now have:
$$
\norm{Xb+ u}_2^2 \leq \varepsilon^\top (Xb +  u) + \rho \big(\tilde{\Lambda}(k) \norm{b}_2 - \sum_{j=k+1}^p \tilde{\lambda}_j \abs{b}_{(j)} \big) + \rho \big(\Lambda(s)\norm{u}_2 - \sum_{j=s+1}^n \lambda_j \abs{u}_{(j)} \big)
$$
We use twice Lemma \ref{lem:inner_product_bound} to bound $\varepsilon^\top Xb$ and $\varepsilon^\top u$ with $(k/2p)^k$ and~$(s/2n)^s$ as respective choices of $\delta_0$, so that with probability $1- \frac{1}{2}\left(\frac{s}{2n}\right)^{s} - \frac{1}{2}\left(\frac{k}{2p}\right)^{k} $:
\begin{align*}
\norm{Xb+ u}_2^2 &\leq H(b) + G(b) + H(u) + G(u)  \\
&\quad + \rho\big(\tilde{\Lambda}(k) \norm{b}_2 - \sum_{j=k+1}^p \tilde{\lambda}_j \abs{b}_{(j)}\big) + \rho\big(\Lambda(s)\norm{u}_2 - \sum_{j=s+1}^n \lambda_j \abs{u}_{(j)}\big)\\
&\leq \frac{\rho}{2}\sum_{j=1}^p \tilde{\lambda}_j \abs{b}_{(j)} + \frac{\rho}{2}\sqrt{k\log(2p/k)}\norm{b}_2 + \rho\big(\tilde{\Lambda}(k)\norm{b}_2 - \sum_{j=k+1}^p \tilde{\lambda}_j \abs{b}_{(j)} \big) \\
& \quad +2\rho\Lambda(s)\norm{u}_2 - \frac{\rho}{2} \sum_{j=s+1}^n \lambda_j \abs{u}_{(j)}\\
&\leq \frac{\rho}{2}4\tilde{\Lambda}(k)\norm{b}_2 - \frac{\rho}{2}\sum_{j=k+1}^p \tilde{\lambda}_j \abs{b}_{(j)}  + 2\rho\Lambda(s)\norm{u}_2 - \frac{\rho}{2} \sum_{j=s+1}^n \lambda_j \abs{u}_{(j)},
\end{align*}
where we use Equation~\eqref{eq:sum_log_inequality} to obtain the last inequality.
The left part of the inequality is positive so
\begin{equation}
\sum_{j=k+1}^p \tilde{\lambda}_j \abs{b}_{(j)} + \sum_{j=s+1}^n \lambda_j \abs{u}_{(j)} \leq 4\tilde{\Lambda}(k) \norm{b}_2 + 4\Lambda(s)\norm{u}_2,
\label{eq:C4}
\end{equation}
and
\begin{equation}
2\norm{Xb+ u}_2^2 \leq 4\rho\tilde{\Lambda}(k) \norm{b}_2 + 4\rho\Lambda(s)\norm{u}_2.
\label{eq:C5}
\end{equation}
Equation~\eqref{eq:C4} together with the Cauchy-Schwarz inequality lead to
\begin{equation}
\label{eq:cone_with_two_slope}
\sum_{j=1}^p \tilde{\lambda}_j \abs{b}_{(j)} + \sum_{j=1}^n \lambda_j \abs{u}_{(j)} \leq 5\tilde{\Lambda}(k) \norm{b}_2 + 5\Lambda(s)\norm{u}_2.
\end{equation}
 Combining the equation above with Equation~\eqref{eq:sum_log_inequality} show that the concatenated estimator is in $\mathcal{C}(k_1,s_1,4)$ with $s_1$ and $k_1$ as in the statement of the theorem (note that~$\tilde{\lambda}_n =~\lambda_n =~\sigma\sqrt{\log 2}$) and so, noting $\kappa$ the capacity constant of Assumption~\ref{Hyp}, Equation~\eqref{eq:C5} leads to:

\begin{align*}
2\norm{Xb+ u}_2^2 &\leq (3+C)^2\frac{\tilde{\Lambda}(k)^2}{2\kappa^2} + \kappa^2 \norm{b}_2^2 + 4\rho^2\frac{\Lambda(s)^2}{\kappa^2} + \kappa^2 \norm{u}_2^2\\
&\leq \frac{C'}{\kappa^2}\left( \tilde{\Lambda}(k)^2 + \Lambda(s)^2 \right) +  \norm{Xb+ u}_2^2,
\end{align*}
where $C' = 4\rho^2 \vee (3+C)^2 / 2$.
 Finally:
 $$
 \norm{Xb+ u}_2^2 \leq \frac{C'}{\kappa^2}\left( \tilde{\Lambda}(k)^2 + \Lambda(s)^2 \right),
 $$
 and
 $$
 \norm{b}_2^2 + \norm{u}_2^2 \leq \frac{C'}{\kappa^4}\left( \tilde{\Lambda}(k)^2 + \Lambda(s)^2 \right).
 $$
 
 \section{Proof of Theorem \ref{THM:FDR}}
 \label{Appen:p2}

In this section, we give the proof of the asymptotic FDR control presented in Theorem \ref{THM:FDR}.
In the following, for a given matrix $A$ and a given subset $T$, $A_T$ denotes the extracted matrix formed by the columns of $A$ with indices in $T$, whereas~$A_{T,\cdot}$ denotes the extracted matrix formed by the rows  of $A$ with indices in~$T$. For vectors, there is no ambiguity. Moreover, $S$ (of cardinal $s$) denotes the support of the true parameter $\mu^\star$.

We first recall some properties on the dual of the sorted $\ell_1$ norm, and also a lemma taken from \cite{slopeminimax} and stated here without proof:

\begin{defn}[\cite{slopeminimax}]
A vector $a\in\R^n$ is said to majorize $b\in\R^n$ (denoted $b \preccurlyeq a$) if they satisfy for all $i\in\{1,\dots,n\}$:
$$
\abs{a}_{(1)} + \cdots + \abs{a}_{(i)} \geq \abs{b}_{(1)} + \cdots + \abs{b}_{(i)}.
$$
\end{defn}

\begin{prop}[\cite{slope}]
Let $J_\lambda$ be the sorted $\ell_1$ norm for a certain non-increasing sequence $\lambda$ of length $n$. The unit ball of the dual norm is:
$$\mathcal{C_\lambda} = \{v\in\R^n: v \preccurlyeq \lambda \}.$$
\end{prop}

\begin{lem}[\cite{slopeminimax}, Lemma A.9]
\label{lem:A9}
Given any constant $\alpha > 1/(1-q)$, suppose $\max\{\alpha s, s + d\} \leq s^\star < n $ for any (deterministic) sequence $d$ that diverges to $\infty$. Let $\zeta_1, \dots, \zeta_{n-s}$ be i.i.d $\mathcal{N}(0,1)$. Then 
$$
(\abs{\zeta}_{(s^\star - s +1)}, \abs{\zeta}_{(s^\star - s +2)}, \dots , \abs{\zeta}_{(n - s)}) \preccurlyeq (\lambda_{s^\star+1}^{\mathrm{BH}}, \lambda_{s^\star+2}^{\mathrm{BH}}, \dots, \lambda_{n}^{\mathrm{BH}})
$$
with probability approaching one.
\end{lem}
We adapt from \cite{slopeminimax} the definition of a resolvent set below, useful to determine the true support of the mean-shift parameters.

\begin{defn}
\label{def:res}
Let $s^\star$ be an integer obeying $s<s^\star <n$. The set~$S^\star (S,s^\star)$ is said to be a resolvent set if it is the union of $S$ and\linebreak the $s^\star - s$ indices corresponding to the largest entries of the error term~$\varepsilon$ restricted to $\bar{S}$.
\end{defn}
Let $c$ be any positive constant and fix $s^\star \geq s(1+c) / (1-q)$ ($q$ being the target FDR level), so that assumptions of Lemma \ref{lem:A9} are satisfied. For clarity, we denote $S^\star = S^\star(S, s^\star)$. 
For a resolvent set $S^\star$ of cardinality $s^\star$, define the reduced minimization as:
\begin{equation}
\label{reduced}
\beta^{S^\star}, \mu^{S^\star} = \argmin_{\beta\in\R^p, \mu\in\R^{s^\star}} \left\{\norm{y - X\beta - I_{S^\star}\mu}_2^2 + 2\rho J_{\tilde{\lambda}}(\beta) + 2\rho J_{\lambda^{[s^\star]}}(\mu) \right\},
 \end{equation}
where $\lambda^{[s^\star]}$ is the beginning (the first $s^\star $ terms) of the sequence of weights in the global problem. 
Note that a resolvent set contains the support of the true parameter $\mu^\star$, so the generalized versions of the main results in Section~\ref{SEC:BOUNDS}, considered in the proof in Appendix~\ref{Appen:p1}, hold.

We want to show that the estimator of the unreduced problem $\hat{\mu}$ has null values for coordinates which indices are not in $S^\star$. Precisely, we will show that $\hat{\mu}= I_{S^\star}\mu^{S^\star}$.
The first order conditions for global and reduced minimisation problem above are respectively:

\begin{empheq}[left=\empheqlbrace]{align}
 X^\top (y-X\hat{\beta}-\hat{\mu}) &\in \rho\partial J_{\tilde{\lambda}} (\hat{\beta})\label{eq:glob1} \\
 y-X\hat{\beta}-\hat{\mu} &\in \rho\partial J_\lambda (\hat{\mu}) \label{eq:glob2}
\end{empheq}
and 
\begin{empheq}[left=\empheqlbrace]{align}
X^\top (y-X\beta^{S^\star}-I_{S^\star}\mu^{S^\star}) & \in \rho\partial J_{\tilde{\lambda}} (\beta^{S^\star}) \label{eq:red1}\\
I_{S^\star}^\top (y-X\beta^{S^\star}-I_{S^\star}\mu^{S^\star}) &\in \rho\partial J_{\lambda^{[s^\star]}} (\mu^{S^\star}) \label{eq:red2}
\end{empheq}
Clearly, Equation~\eqref{eq:red1} leads to Equation~\eqref{eq:glob1} taking $\hat{\beta} = \beta^{S^\star}$ and $\hat{\mu} =~I_{S^\star}\mu^{S^\star}.$ We must now show that this choice of $\hat{\beta}$ and $\hat{\mu}$ satisfies Equation~\eqref{eq:glob2}.

First, $y-X\hat{\beta}-\hat{\mu}$ must be in the unit ball of the dual norm, that is~${y-X\hat{\beta}-\hat{\mu} \preccurlyeq \rho\lambda }$. Because $y-X\hat{\beta}-\hat{\mu}\in\R^n$ is the concatenation of~${I_{S^\star}^\top (y-X\hat{\beta}-\hat{\mu})}$ and $I_{\bar{S^\star}}^\top (y-X\hat{\beta}-\hat{\mu})$, we must check that $S^\star$ satisfy:
$$
I_{\overline{S^\star}}^\top (y-X\hat{\beta}^{S^\star}-I_{S^\star}\hat{\mu}^{S^\star}) \preccurlyeq \rho\lambda^{-[s^\star]},
$$
where $\lambda^{-[s^\star]}$ is the end of the sequence in the global problem (omitting the first $s^\star $ terms).
If so, noting that if $a_1 \preccurlyeq b_1$ and $a_2 \preccurlyeq b_2$ then $a\preccurlyeq b$ (with $a$ and $b$ being the respective concatenation of $a_1,a_2$ and $b_1,b_2$) and combining it with Equation~\eqref{eq:red2} will lead to the belonging at the unit ball of the dual norm.

Equivalently, we must check that 
$$
y_{\overline{S^\star}} - X_{\overline{S^\star}, \cdot}\beta^{S^\star} \preccurlyeq \rho\lambda^{-[s^\star]},
$$
or also 
\begin{equation}
\label{dual}
X_{\overline{S^\star}, \cdot}(\beta^\star - \beta^{S^\star}) + I_{\overline{S^\star}, \cdot} \varepsilon \preccurlyeq \rho\lambda^{-[s^\star]}
\end{equation}
Lemma \ref{lem:A9}, together with the definition of the resolvent set $S^\star$ given in Definition \ref{def:res}, allows us to handle the second term to obtain, with probability tending to one:
$$
 I_{\overline{S^\star}, \cdot} \varepsilon \preccurlyeq (\lambda^{\mathrm{BH}})^{-[s^\star]} \preccurlyeq \rho (\lambda^{\mathrm{BH}})^{-[s^\star]}.
$$
It remains to control the term $X_{\overline{S^\star}, \cdot}(\beta^\star - \beta^{S^\star})  $. For our purpose, it is sufficient to show that $\Vert X_{\overline{S^\star}, \cdot}(\beta^\star - \beta^{S^\star})\Vert_{\infty}$ tends to zero when $n$ goes to infinity, because in this case we would have $X_{\overline{S^\star}, \cdot}(\beta^\star - \beta^{S^\star}) \preccurlyeq \rho\epsilon (\lambda^{\mathrm{BH}})^{-[s^\star]} $ if $n$ is large enough. Thus, let $i\in\{1,\dots, n\}$ and $x_i$ the $i^{th}$ row of $X$, then we have:
$$
\vert\langle x_i, \beta^\star - \beta^{S^\star} \rangle \vert \leq  \sum_{j=1}^p \vert x_{i,j}\vert \vert \beta^\star - \beta^{S^\star}\vert_{j} \leq \frac{M}{\sqrt{n}} \Vert\beta^\star - \beta^{S^\star}\Vert_1.
$$
Now we distinguish the three cases. For Procedure~\ref{eq:procedure_no_beta}, we do not assume sparsity on $\beta$ so we rely on the Cauchy-Schwarz inequality to obtain
\begin{equation*}
\vert\langle x_i, \beta^\star - \beta^{S^\star} \rangle\vert \leq \frac{M}{\sqrt{n}}\sqrt{p} \Vert\beta^\star - \beta^{S^\star}\Vert_2.
\end{equation*}
For Procedure~\ref{eq:2}, Equation~\eqref{eq:cone_lasso_slope} lead to the bound
\begin{equation*}
\vert\langle x_i, \beta^\star - \beta^{S^\star} \rangle\vert \leq \frac{M}{\sqrt{n}}C\big(k\log p \vee s\log(2en/s)\big) \big(\Vert\beta^\star - \beta^{S^\star}\Vert_2 \vee \Vert\mu^\star - \mu^{S^\star}\Vert_2\big),
\end{equation*}
with $C$ being some positive constant.
For Procedure~\ref{eq:procedure_with_two_slopes}, Equation~\eqref{eq:cone_with_two_slope} lead to the bound
\begin{equation*}
\vert\langle x_i, \beta^\star - \beta^{S^\star} \rangle\vert \leq \frac{M}{\sqrt{n}}C'\big(k\log (2ep/k) \vee s\log(2en/s)\big) \big(\Vert\beta^\star - \beta^{S^\star}\Vert_2 \vee \Vert\mu^\star - \mu^{S^\star}\Vert_2\big),
\end{equation*}
with $C'$ being some positive constant.

Therefore the coordinates are uniformly bounded by a quantity tending to zero in each of the three cases of the theorem, thanks to the upper bounds obtained in Section~\ref{SEC:BOUNDS}. Now it is sufficient to choose $n$ such that $\abs{\langle x_i, \beta^\star - \beta^{S^\star} \rangle} \leq \rho\epsilon\lambda^{\mathrm{BH}}_n$ (it is important to notice that the right term does not depend on $n$ and equals to $\rho\epsilon\Phi^{-1}(1-q/2)$) to finally obtain Equation~\eqref{dual}.
Note that Equation~\eqref{dual} is the necessary condition for $y-X\beta^{S^\star}-I_{S^\star}\mu^{S^\star}$ to be feasible (meaning in the unit ball $\mathcal{C}_\lambda$ of the dual norme of $J_\lambda$) but this is also sufficient for being in the subdifferential because
$$
\partial J_\lambda (x) = \{ \omega\in\mathcal{C}_\lambda \;: \; \langle \omega, x \rangle = J_\lambda (x) \}, 
$$
and as we have, due to Equation\eqref{eq:red2}:
$$
\langle I_{S^\star}^\top (y-X\beta^{S^\star}-I_{S^\star}\mu^{S^\star}), \mu^{S^\star} \rangle = J_{\lambda^{[s^\star]}}(\mu^{S^\star}),
$$
then: 
$$
\langle y-X\beta^{S^\star}-I_{S^\star}\mu^{S^\star}, I_{S^\star}\mu^{S^\star} \rangle = J_{\lambda}(I_{S^\star} \mu^{S^\star}).
$$
Therefore, with probability tending to one, $\hat{\mu} = I_{S^\ast}\mu^{S^\star}$ and in particular 

\begin{equation}
\label{subset_fdr}
\supp(\hat{\mu}) \subset S^\star.
\end{equation}

We now show that the support of $\hat{\mu}$ contains the support of $\mu^\star$.
Considering Equation~\eqref{eq:red2} we have in particular the belonging to the unit ball of the dual norm, that is to say: 
$$
I_{S^\star}^\top (y-X\beta^{S^\star}-I_{S^\star}\mu^{S^\star}) \preccurlyeq \rho\lambda^{[s^\star]}.
$$
In particular we have
$$
\Vert I_{S^\star}^\top (y-X\beta^{S^\star}-I_{S^\star}\mu^{S^\star}) \Vert_\infty \leq \rho\lambda_1.
$$
Having $y = X\beta^\star + \mu^\star + \varepsilon = X\beta^\star + I_{S^\star}(\mu^\star)_{S^\star} + \varepsilon$, the inequality above re-writes as:
$$
\Vert X_{S^\star,\cdot}(\beta^\star - \beta^{S^\star}) + \mu_{S^\star}^\star - \mu^{S^\star} + I_{S^\star}^\top \varepsilon \Vert_\infty \leq \rho\lambda_1.
$$
By the triangle inequality, we obtain:
$$
\Vert \mu_{S^\star}^\star - \mu^{S^\star} \Vert_\infty \leq \rho\lambda_1 + \Vert X_{S^\star,\cdot}(\beta^\star - \beta^{S^\star}) + I_{S^\star}^\top \varepsilon   \Vert_\infty \leq \rho\lambda_1 + \Vert X_{S^\star,\cdot}(\beta^\star - \beta^{S^\star}) \Vert_\infty + \Vert I_{S^\star}^\top \varepsilon   \Vert_\infty 
$$
Now, we already said that we have $\Vert X_{S^\star,\cdot}(\beta^\star - \beta^{S^\star}) \Vert_\infty \leq \rho\epsilon\lambda^{\mathrm{BH}}_n \leq  \rho\epsilon\lambda^{\mathrm{BH}}_1 $, and using the standard bound on the norm of a Gaussian noise (see Lemma~\ref{lm:Gauss2}), we also have, with probability tending to one (precisely with probability $1/n$):
 $$
 \Vert I_{S^\star}^\top \varepsilon   \Vert_\infty \leq \Vert \varepsilon   \Vert_\infty\leq 2\sigma \sqrt{\log n}.
 $$ 
 Combining the previous inequalities leads to:
 $$
 \Vert \mu_{S^\star}^\star - \mu^{S^\star} \Vert_\infty \leq \rho\lambda_1 + \rho\epsilon\lambda^{\mathrm{BH}}_1 + 2\sigma\sqrt{\log n} = \rho(1+2\epsilon)\lambda^{\mathrm{BH}}_1 + 2\sigma\sqrt{\log n}.
 $$
 A standard bound for the Gaussian quantile function gives $\lambda^{\mathrm{BH}}_1 \leq \sigma \sqrt{2\log (2n/q)}$, so with $q\geq 2/n$ (this is quite artificial, $q$ is generally more than $0.01$) we obtain:
 $$
 \Vert (\mu^\star)_{S^\star} - \mu^{S^\star} \Vert_\infty \leq (1+\rho(1+2\epsilon))2\sigma\sqrt{\log n}.
 $$
 Therefore, because the entries of $\mu^\star$ are of absolute values greater than the right bound of the above inequality we obtain:
 $$
 S \subset \supp ((\mu^\star)_{S^\star}) \subset \supp (\mu^{S^\star}) \subset \supp (\hat{\mu}),
 $$
 and so the Power tends to one.
 
 It remains to show the FDR control, using Equation~\eqref{subset_fdr}. 
 Define the False Discovery Proportion (FDP) as $V/(R\vee 1)$, where $R$ and $V$ are defined in Equation~\eqref{eq:fdr}.  Because of the inclusion $S \subset  \supp (\hat{\mu})$, the FDP is~${(R-s)/R = 1-s/R}$ with probability tending to one. According to Equation~\eqref{subset_fdr} and the assumption on $s^\star$,
 $$
 \mathrm{FDP} = 1-\frac{s}{R} \leq 1-\frac{s}{s^\star} \leq 1-\frac{1-q}{1+c} = \frac{q+c}{1+c} \leq q+c,
 $$
 with probability tending to one.
 In expectation, and with $n$ tending to infinity, we obtain:
 $$
 \limsup_{n\rightarrow +\infty}\mathrm{FDR}(\hat{\mu}) \leq q+c,
 $$
and $c$ being arbitrarily close to zero leads to the conclusion.

\section{Supplementary simulations}
\label{appensim}

We gather here some extra-simulations in low dimension to complete the ones from Section~\ref{subsec:simu} with higher FDR level or/and higher correlation level for the design matrix. As it is our particular interest, we focus on experiments with outliers of weak magnitudes.

Figure~\ref{fig:lowdim_appen0} below is the same as in Section~\ref{subsec:simu} for settting~1, excepted that the correlation of the design matrix is now higher ($0.8$). 
Results are similar to those obtained in Section \ref{subsec:simu}, that is even with a higher correlation, E-SLOPE is able to make much more discoveries while keeping the FDR under control.
\begin{figure}[htbp]
\centering
\includegraphics[width=\linewidth]{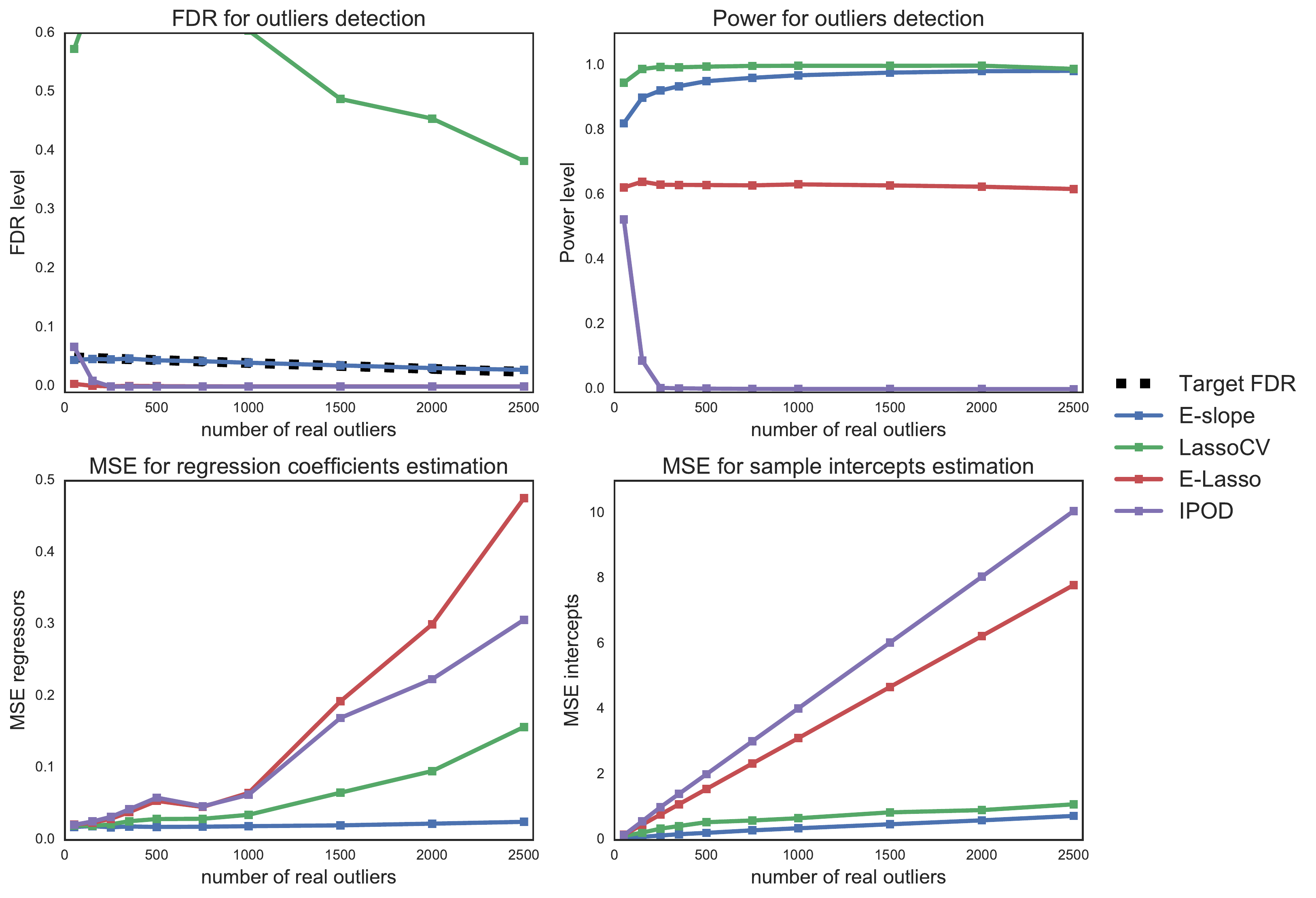}
\caption{Results for simulation Setting~1 with low-magnitude outliers and correlation $\rho = 0.8$. First row gives the FDR (left) and power (right) of each considered procedure for outliers discoveries. 
Second row gives the MSE for regressors (left) and intercepts (right). 
E-SLOPE gives perfect power, is the only one to respect the required FDR, and provides the best MSEs.}
\label{fig:lowdim_appen0}
\end{figure}

Figures~\ref{fig:lowdim_appen1},~\ref{fig:lowdim_appen2} below gather the results of simulations in setting~1 of Section~\ref{sec:simu} with other target FDR (here $10\%$ and $20\%$) and for both moderate and high correlation (respectively $0.4$ and $0.8$). E-LASSO and IPOD do not depend on the target FDR so they are not plotted again. The results confirm the fact that E-SLOPE provides a high power together with a FDR control for a wide range of target FDR level.
The left (resp. right) columns contain the results for $\alpha = 10\%$ (resp. $\alpha = 20\%$) as indicated by the straight lines.

\begin{figure}[H]
\centering
\includegraphics[width=\linewidth]{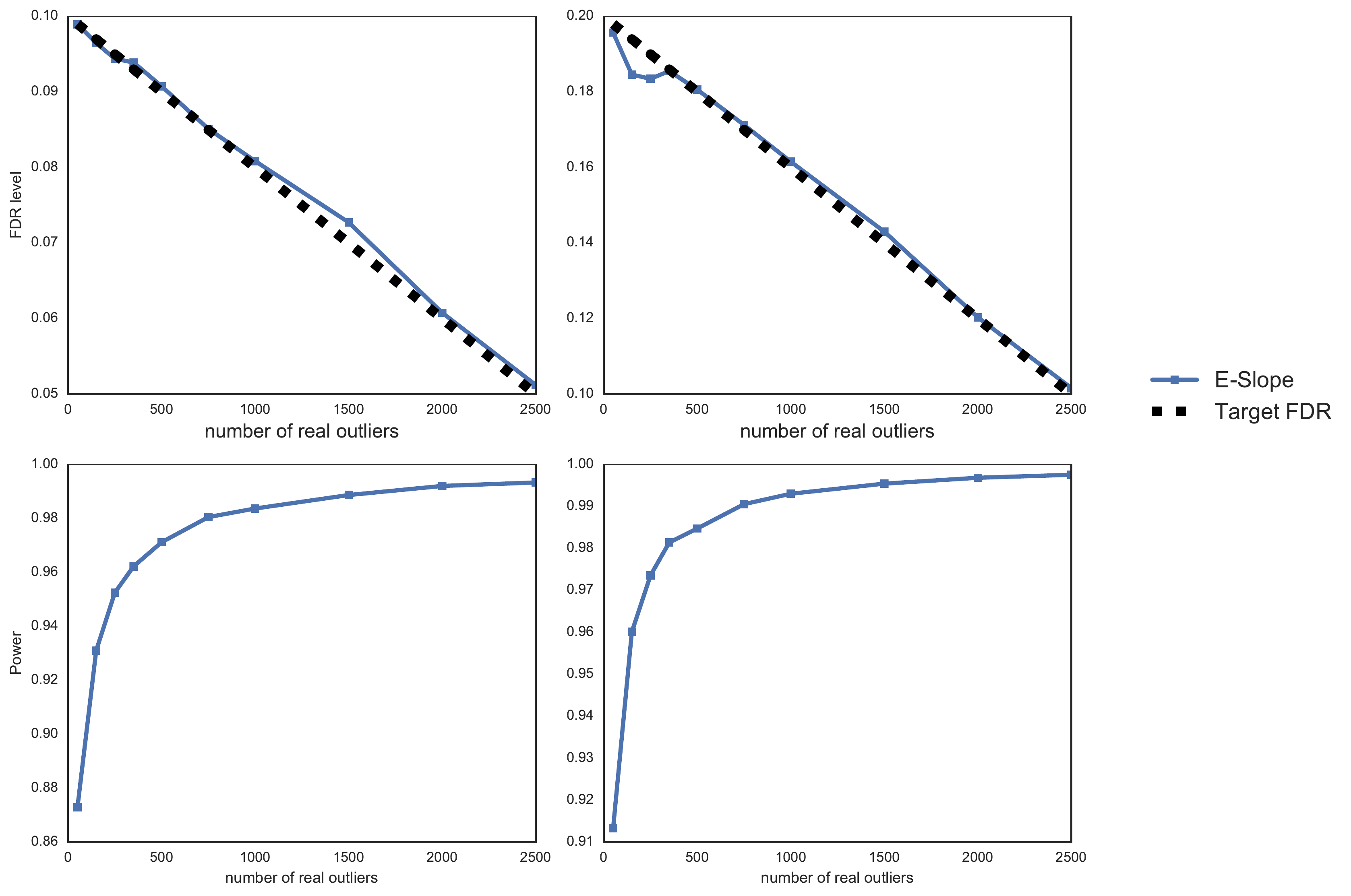}
\caption{Results for simulation Setting~1 with low-magnitude outliers, correlation $\rho = 0.4$. Left column gives the FDR (top) and power (bottom) for E-SLOPE with target FDR $\alpha = 10\%$.
Right column gives the FDR (top) and power (bottom) for E-SLOPE with target FDR $\alpha = 20\%$.}
\label{fig:lowdim_appen1}
\end{figure}

\begin{figure}[H]
\centering
\includegraphics[width=\linewidth]{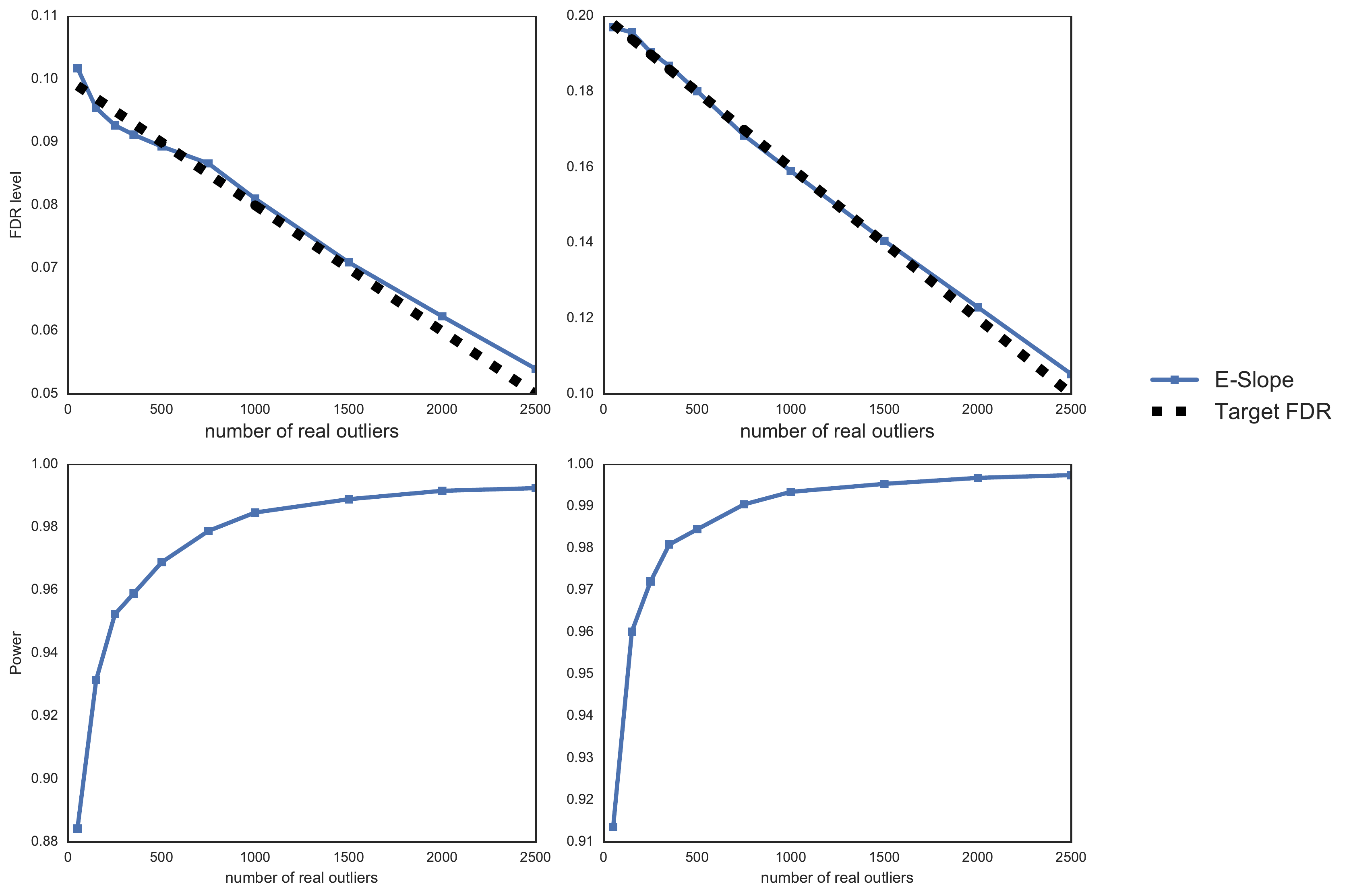}
\caption{Results for simulation Setting~1 with low-magnitude outliers, correlation $\rho = 0.8$. Left column gives the FDR (top) and power (bottom) for E-SLOPE with target FDR $\alpha = 10\%$.
Right column gives the FDR (top) and power (bottom) for E-SLOPE with target FDR $\alpha = 20\%$.}
\label{fig:lowdim_appen2}
\end{figure}

\end{appendix}

\bibliographystyle{abbrv}


\end{document}